\documentclass[12pt,letterpaper]{article}

\def\showauthornotes{0}
\def\showtableofcontents{1}
\def\showkeys{0}
\def\showdraftbox{0}
\def\showcolorlinks{1}
\def\usemicrotype{1}
\def\showfixme{0}

\usepackage{etex}

\usepackage[l2tabu, orthodox]{nag}

\usepackage{xspace,enumerate}

\usepackage[dvipsnames]{xcolor}

\usepackage[T1]{fontenc}
\usepackage[full]{textcomp}

\usepackage[american]{babel}

\usepackage{mathtools}

\usepackage{amsthm}

\newtheorem{theorem}{Theorem}[section]
\newtheorem*{theorem*}{Theorem}

\newtheorem{proposition}[theorem]{Proposition}
\newtheorem*{proposition*}{Proposition}
\newtheorem{lemma}[theorem]{Lemma}
\newtheorem*{lemma*}{Lemma}
\newtheorem{corollary}[theorem]{Corollary}
\newtheorem*{conjecture*}{Conjecture}
\newtheorem{fact}[theorem]{Fact}
\newtheorem*{fact*}{Fact}

\newtheorem*{hypothesis*}{Hypothesis}

\theoremstyle{definition}

\newtheorem{algorithm}[theorem]{Algorithm}
\newtheorem{problem}[theorem]{Problem}

\theoremstyle{remark}
\newtheorem{claim}[theorem]{Claim}
\newtheorem*{claim*}{Claim}

\newtheorem*{remark*}{Remark}
\newtheorem{observation}[theorem]{Observation}
\newtheorem*{observation*}{Observation}

\usepackage[
letterpaper,
top=0.7in,
bottom=0.9in,
left=1in,
right=1in]{geometry}

\usepackage{newpxtext} %
\usepackage{textcomp} %
\usepackage[varg,bigdelims]{newpxmath}
\usepackage[scr=rsfso]{mathalfa}%
\usepackage{bm} %
\let\mathbb\varmathbb

\ifnum\showkeys=1
\usepackage[color]{showkeys}
\fi

\ifnum\showcolorlinks=1
\usepackage[
pagebackref,
colorlinks=true,
urlcolor=blue,
linkcolor=blue,
citecolor=OliveGreen,
]{hyperref}
\fi

\ifnum\showcolorlinks=0
\usepackage[
pagebackref,
colorlinks=false,
pdfborder={0 0 0}
]{hyperref}
\fi

\usepackage{prettyref}

\newcommand{\savehyperref}[2]{\texorpdfstring{\hyperref[#1]{#2}}{#2}}

\newrefformat{eq}{\savehyperref{#1}{\textup{(\ref*{#1})}}}
\newrefformat{eqn}{\savehyperref{#1}{\textup{(\ref*{#1})}}}
\newrefformat{lem}{\savehyperref{#1}{Lemma~\ref*{#1}}}
\newrefformat{obs}{\savehyperref{#1}{Observation~\ref*{#1}}}
\newrefformat{def}{\savehyperref{#1}{Definition~\ref*{#1}}}
\newrefformat{thm}{\savehyperref{#1}{Theorem~\ref*{#1}}}
\newrefformat{cor}{\savehyperref{#1}{Corollary~\ref*{#1}}}
\newrefformat{cha}{\savehyperref{#1}{Chapter~\ref*{#1}}}
\newrefformat{sec}{\savehyperref{#1}{Section~\ref*{#1}}}
\newrefformat{app}{\savehyperref{#1}{Appendix~\ref*{#1}}}
\newrefformat{tab}{\savehyperref{#1}{Table~\ref*{#1}}}
\newrefformat{fig}{\savehyperref{#1}{Figure~\ref*{#1}}}
\newrefformat{hyp}{\savehyperref{#1}{Hypothesis~\ref*{#1}}}
\newrefformat{alg}{\savehyperref{#1}{Algorithm~\ref*{#1}}}
\newrefformat{rem}{\savehyperref{#1}{Remark~\ref*{#1}}}
\newrefformat{item}{\savehyperref{#1}{Item~\ref*{#1}}}
\newrefformat{step}{\savehyperref{#1}{step~\ref*{#1}}}
\newrefformat{conj}{\savehyperref{#1}{Conjecture~\ref*{#1}}}
\newrefformat{fact}{\savehyperref{#1}{Fact~\ref*{#1}}}
\newrefformat{prop}{\savehyperref{#1}{Proposition~\ref*{#1}}}
\newrefformat{prob}{\savehyperref{#1}{Problem~\ref*{#1}}}
\newrefformat{claim}{\savehyperref{#1}{Claim~\ref*{#1}}}
\newrefformat{relax}{\savehyperref{#1}{Relaxation~\ref*{#1}}}
\newrefformat{red}{\savehyperref{#1}{Reduction~\ref*{#1}}}
\newrefformat{part}{\savehyperref{#1}{Part~\ref*{#1}}}

\newcommand{\Sref}[1]{\hyperref[#1]{\S\ref*{#1}}}

\usepackage{nicefrac}

\ifnum\usemicrotype=1
\usepackage{microtype}
\fi

\ifnum\showauthornotes=1
\newcommand{\Authornote}[2]{{\sffamily\small\color{red}{[#1: #2]}}}
\newcommand{\Authornotecolored}[3]{{\sffamily\small\color{#1}{[#2: #3]}}}
\newcommand{\Authorcomment}[2]{{\sffamily\small\color{gray}{[#1: #2]}}}
\newcommand{\Authorstartcomment}[1]{\sffamily\small\color{gray}[#1: }

\newcommand{\Authorfnote}[2]{\footnote{\color{red}{#1: #2}}}
\newcommand{\Authorfixme}[1]{\Authornote{#1}{\textbf{??}}}
\newcommand{\Authormarginmark}[1]{\marginpar{\textcolor{red}{\fbox{\Large #1:!}}}}
\else
\newcommand{\Authornote}[2]{}
\newcommand{\Authornotecolored}[3]{}
\newcommand{\Authorcomment}[2]{}
\newcommand{\Authorstartcomment}[1]{}

\newcommand{\Authorfnote}[2]{}
\newcommand{\Authorfixme}[1]{}
\newcommand{\Authormarginmark}[1]{}
\fi

\newcommand{\Dnote}{\Authornote{D}}

\definecolor{forestgreen(traditional)}{rgb}{0.0, 0.27, 0.13}

\ifnum\showfixme=0

\fi

\usepackage{boxedminipage}

\newcommand{\paren}[1]{(#1)}

\newcommand{\card}[1]{\lvert#1\rvert}

\newcommand{\norm}[1]{\lVert#1\rVert}

\newcommand{\iprod}[1]{\langle#1\rangle}

\newcommand{\Esymb}{\mathbb{E}}
\newcommand{\Psymb}{\mathbb{P}}
\newcommand{\Vsymb}{\mathbb{V}}

\DeclareMathOperator*{\E}{\Esymb}
\DeclareMathOperator*{\Var}{\Vsymb}
\DeclareMathOperator*{\ProbOp}{\Psymb}

\renewcommand{\Pr}{\ProbOp}

\newcommand{\tensor}{\otimes}

\newcommand{\textparen}[1]{\text{(#1)}}

\ifx\because\undefined
\newcommand{\because}[1]{\textparen{because #1}}
\else
\renewcommand{\because}[1]{\textparen{because #1}}
\fi

\newcommand{\super}[2]{#1^{\paren{#2}}}

\newcommand{\inv}[1]{{#1^{-1}}}

\newcommand{\defeq}{\stackrel{\mathrm{def}}=}

\newcommand{\mper}{\,.}
\newcommand{\mcom}{\,,}

\newcommand\bdot\bullet

\ifx\mathds\undefined %
\DeclareMathOperator{\Ind}{\mathbb{I}}
\else
\DeclareMathOperator{\Ind}{\mathds 1}}
\fi

\DeclareMathOperator{\poly}{poly}

\DeclareMathOperator{\argmax}{argmax}

\DeclareMathOperator{\rank}{rank}

\newcommand{\R}{\mathbb R}

\newcommand{\cC}{\mathcal C}
\newcommand{\cD}{\mathcal D}
\newcommand{\cE}{\mathcal E}

\newcommand{\cG}{\mathcal G}

\newcommand{\cN}{\mathcal N}

\newcommand{\cS}{\mathcal S}

\newcommand{\cX}{\mathcal X}

\renewcommand{\le}{\leqslant}

\renewcommand{\ge}{\geqslant}

\ifnum\showdraftbox=1
\newcommand{\draftbox}{\begin{center}
  \fbox{%
    \begin{minipage}{2in}%
      \begin{center}%
          \Large\textsc{Working Draft}\\%
        Please do not distribute%
      \end{center}%
    \end{minipage}%
  }%
\end{center}
\vspace{0.2cm}}
\else
\newcommand{\draftbox}{}
\fi

\let\epsilon=\varepsilon

\numberwithin{equation}{section}

\newcommand\MYcurrentlabel{xxx}

\newcommand{\MYstore}[2]{%
  \global\expandafter \def \csname MYMEMORY #1 \endcsname{#2}%
}

\newcommand{\MYload}[1]{%
  \csname MYMEMORY #1 \endcsname%
}

\newcommand{\MYnewlabel}[1]{%
  \renewcommand\MYcurrentlabel{#1}%
  \MYoldlabel{#1}%
}

\newcommand{\MYdummylabel}[1]{}

\newcommand{\torestate}[1]{%
  \let\MYoldlabel\label%
  \let\label\MYnewlabel%
  #1%
  \MYstore{\MYcurrentlabel}{#1}%
  \let\label\MYoldlabel%
}

\newcommand{\restatetheorem}[1]{%
  \let\MYoldlabel\label
  \let\label\MYdummylabel
  \begin{theorem*}[Restatement of \prettyref{#1}]
    \MYload{#1}
  \end{theorem*}
  \let\label\MYoldlabel
}

\newcommand{\restatelemma}[1]{%
  \let\MYoldlabel\label
  \let\label\MYdummylabel
  \begin{lemma*}[Restatement of \prettyref{#1}]
    \MYload{#1}
  \end{lemma*}
  \let\label\MYoldlabel
}

\newcommand{\restateprop}[1]{%
  \let\MYoldlabel\label
  \let\label\MYdummylabel
  \begin{proposition*}[Restatement of \prettyref{#1}]
    \MYload{#1}
  \end{proposition*}
  \let\label\MYoldlabel
}

\newcommand{\restatefact}[1]{%
  \let\MYoldlabel\label
  \let\label\MYdummylabel
  \begin{fact*}[Restatement of \prettyref{#1}]
    \MYload{#1}
  \end{fact*}
  \let\label\MYoldlabel
}

\newcommand{\restate}[1]{%
  \let\MYoldlabel\label
  \let\label\MYdummylabel
  \MYload{#1}
  \let\label\MYoldlabel
}

\newcommand{\addreferencesection}{
  \phantomsection
  \addcontentsline{toc}{section}{References}
}

\newcommand{\e}{\epsilon}
\newcommand{\eps}{\epsilon}

\let\origparagraph\paragraph
\renewcommand{\paragraph}[1]{\origparagraph{#1.}}

\allowdisplaybreaks

\usepackage{paralist}

\usepackage{comment}
\usepackage{relsize}
\usepackage[font=footnotesize]{caption}

\DeclareMathOperator{\Span}{Span}

\DeclareMathOperator{\Id}{\mathrm{Id}}

\DeclareUrlCommand\email{}

\DeclareMathOperator{\tO}{\tilde{O}}

\newcommand{\bT}{\mathbf{T}}

\let\pref=\prettyref

\newcommand{\bS}{\mathbf S}
\newcommand{\bE}{\mathbf E}
\newcommand{\bM}{\mathbf M}

\renewcommand{\tO}{{\widetilde O}}
\newcommand{\tOmega}{{\widetilde \Omega}}

\DeclareMathOperator{\argmin}{argmin}

\title{Fast and robust tensor decomposition with applications to dictionary learning}

\author{%
    Tselil Schramm\thanks{UC Berkeley, \protect\email{tschramm@cs.berkeley.edu}.
    T. S. is supported by an NSF Graduate Research Fellowship (NSF award no. 1106400).}
\and
David Steurer\thanks{Cornell University, \protect\email{dsteurer@cs.cornell.edu}.
  D. S. is supported by a Microsoft Research Fellowship, a Alfred P. Sloan Fellowship, NSF awards (CCF-1408673,CCF-1412958,CCF-1350196), and the Simons Collaboration for Algorithms and Geometry.}
  }

\begin{document}

\maketitle
\draftbox
\thispagestyle{empty}

\begin{abstract}
  We develop fast spectral algorithms for tensor decomposition that match the robustness guarantees of the best known polynomial-time algorithms for this problem based on the sum-of-squares (SOS) semidefinite programming hierarchy.

  Our algorithms can decompose a 4-tensor with $n$-dimensional orthonormal components in the presence of error with constant spectral norm (when viewed as an $n^2$-by-$n^2$ matrix).
  The running time is $n^5$ which is close to linear in the input size $n^4$.

  We also obtain algorithms with similar running time to learn sparsely-used orthogonal dictionaries even when feature representations have constant relative sparsity and non-independent coordinates.

  The only previous polynomial-time algorithms to solve these problem are based on solving large semidefinite programs.
  In contrast, our algorithms are easy to implement directly and are based on spectral projections and tensor-mode rearrangements.

  Or work is inspired by recent of Hopkins, Schramm, Shi, and Steurer (STOC'16) that shows how fast spectral algorithms can achieve the guarantees of SOS for average-case problems.
  In this work, we introduce general techniques to capture the guarantees of SOS for worst-case problems.
\end{abstract}

\clearpage

\ifnum\showtableofcontents=1
{
\tableofcontents
\thispagestyle{empty}
 }
\fi

\clearpage

\setcounter{page}{1}

\section{Introduction}
\label{sec:introduction}

Tensor decomposition is the following basic inverse problem:
Given a $k$-th order tensor $T\in (\R^d)^{\otimes k}$ of the form
\begin{equation}
  \label{eq:1}
  T=\sum_{i=1}^n a_i^{\otimes k} + E,
\end{equation}
we aim to approximately recover one or all of the unknown components $a_1,\ldots,a_n\in \R^d$.
The goal is to develop algorithms that can solve this problem efficiently under the weakest possible assumptions on the order $k$, the components $a_1,\ldots,a_n$, and the error $E$.

Tensor decomposition is studied extensively across many disciplines including machine learning and signal processing.
It is a powerful primitive for solving a wide range of other inverse / learning problems, for example: blind source separation / independent component analysis \cite{DBLP:journals/tsp/LathauwerCC07}, learning phylogenetic trees and hidden Markov models \cite{DBLP:conf/stoc/MosselR05}, mixtures of Gaussians \cite{MR3385380-Hsu13}, topic models \cite{DBLP:conf/nips/AnandkumarFHKL12}, dictionary learning \cite{DBLP:conf/stoc/BarakKS15,DBLP:conf/focs/MaSS16}, and noisy-or Bayes nets \cite{DBLP:journals/corr/AroraGMR16}.

A classical algorithm based on simultaneous diagonalization \cite{harshman1970foundations,DBLP:conf/eusipco/LathauwerMV96} (often attributed to R. Jennrich) can decompose the input tensor \pref{eq:1} when the components are linearly independent, there is no error, and the order of the tensor is at least $3$.
Current research on algorithms for tensor decomposition aims to improve over the guarantees of this classical algorithm in two important ways:

\begin{description}
\item[Overcomplete tensors:]
  What conditions allow us to decompose tensors when components are linearly dependent?

\item[Robust decomposition:]
  What kind of errors can efficient decomposition algorithms tolerate?
  Can we tolerate errors $E$ with ``magnitude'' comparable to the low-rank part $\sum_{i=1}^n a_{i}^{\otimes k}$?
\end{description}

The focus of this work is on robustness.
There are two ways in which errors arise in applications of tensor decomposition.
The first is due to finite samples.
For example, in some applications $T$ is the empirical $k$-th moment of some distribution and the error $E$ accounts for the difference between the empirical moment and actual moment (``population moment'').
Errors of this kind can be made smaller at the expense of requiring a larger number of samples from the distribution.
Therefore, robustness of decomposition algorithms helps with reducing sample complexity.

Another way in which errors arise is from modeling errors (``systematic errors'').
These kinds of errors are more severe because they cannot be reduced by taking larger samples.
Two important applications of tensor decomposition with such errors are learning Noisy-or Bayes networks \cite{DBLP:journals/corr/AroraGMR16} and sparsely-used dictionaries \cite{DBLP:conf/stoc/BarakKS15}.
For noisy-or networks, the errors arise due to non-linearities in the model.
For sparsely-used dictionaries, the errors arise due to unknown correlations in the distribution of sparse feature representations.
These examples show that robust tensor decomposition allows us to capture a wider range of models.

Robustness guarantees for tensor decomposition algorithms have been studied extensively (e.g., the work on tensor power iteration \cite{DBLP:journals/jmlr/AnandkumarGHKT14}).
The polynomial-time algorithm with the best known robustness guarantees for tensor decomposition  \cite{DBLP:conf/stoc/BarakKS15,DBLP:conf/focs/MaSS16} are based on the sum-of-squares (SOS) method, a powerful meta-algorithm for polynomial optimization problems based on semidefinite programming relaxations.
Unfortunately, these algorithms are far from practical and have polynomial running times with large exponents.
The goal of this work is to develop practical tensor decomposition algorithms with robustness guarantees close to those of SOS-based algorithms.

For the sake of exposition, we consider the case that the components $a_1,\ldots,a_n\in \R^d$ of the input tensor $T$ are orthonormal.
(Through standard reductions which we explain later, most of our results also apply to components that are spectrally close to orthonormal or at least linearly independent.)
It turns out that the robustness guarantees that SOS achieves for the case that $T$ is a 4-tensor are significantly stronger than its guarantees for 3-tensors.
These stronger guarantees are crucial for applications like dictionary learning.
(It also turns out that for 3-tensors, an analysis of Jennrich's aforementioned algorithm using matrix concentration inequalities gives robustness guarantees that are similar to those of SOS \cite{DBLP:conf/focs/MaSS16,DBLP:journals/corr/AroraGMR16}.)

In this work, we develop an easy-to-implement, randomized spectral algorithm to decompose 4-tensors with orthonormal components even when the error tensor $E$ has small but constant spectral norm as a $d^2$-by-$d^2$ matrix.
This robustness guarantee is qualitatively optimal with respect to this norm in the sense that an error tensor $E$ with constant spectral norm (as a $d^2$-by-$d^2$ matrix) could change each component by a constant proportion of its norm.
To the best of our knowledge, the only previous algorithms with this kind of robustness guarantee are based on SOS.\footnote{We remark that the aforementioned analysis \cite{DBLP:conf/focs/MaSS16} of Jennrich's algorithm can tolerate errors $E$ if its spectral norm as a non-square $d^3$-by-$d$ matrix is constant.
However, this norm of $E$ can be larger by a $\sqrt d $ factor than its spectral norm as a square matrix.}
Our algorithm runs in time $d^{2+\omega}\le d^{4.373}$ using fast matrix multiplication.
Even without fast matrix multiplication, our running time of $d^{5}$ is close to linear in the size of the input $d^4$ and significantly faster than the running time of SOS.
As we will discuss later, an extension of this algorithm allows us to solve instances of dictionary learning that previously could provably be solved only by SOS.

A related previous work \cite{DBLP:conf/stoc/HopkinsSSS16} also studied the question how to achieve similar guarantees for tensor decomposition as SOS using just spectral algorithms.
Our algorithms follow the same general strategy as the algorithms in this prior work:
In an algorithm based on the SOS semidefinite program, one solves a convex programming relaxation to obtain a ``proxy'' for a true, integral solution (such as a component of the tensor).
Because the program is a relaxation, one must then process or ``round'' the relaxation or proxy into a true solution.
In our algorithms, as in \cite{DBLP:conf/stoc/HopkinsSSS16}, instead of finding solutions to SOS semidefinite programs, the algorithms find ``proxy objects'' that behave in similar ways with respect to the rounding procedures used by SOS-based algorithms.
Since these rounding procedures tend to be quite simple, there is hope that generating proxy objects that ``fool'' these procedures is computationally more efficient than solving general semidefinite programs.

However, our algorithmic techniques for finding these ``proxy objects'' differ significantly from those in prior work.
The reason is that many of the techniques in \cite{DBLP:conf/stoc/HopkinsSSS16}, e.g., concentration inequalities for matrix-valued polynomials, are tailored to average-case problems and therefore do not apply in our setting because we do not make distributional assumption about the errors $E$.

The basic version of our algorithm is specified by a sequence of convex sets $\cX_1,\ldots,\cX_r\subseteq (\R^d)^{\otimes 4}$ of $4$-tensors and proceeds as follows:
\begin{quote}
Given a $4$-tensor $T\in(\R^d)^{\otimes 4}$, compute iterative projections $\super T 1,\ldots,\super T r$ to the convex sets $\cX_1,\ldots,\cX_r$ (with respect to euclidean norm) and apply Jennrich's algorithm on $\super T r$.
\end{quote}
It turns out that the SOS-based algorithm correspond to the case that $r=1$ and $\cX_1=\cX_{\mathrm{SOS}}$ is the feasible region of a large semidefinite program.
For our fast algorithm, $\cX_1,\ldots,\cX_r$ are simpler sets defined in terms of singular values or eigenvalues of matrix reshapings of tensors.
Therefore, projections boil down to fast eigenvector computations.
We choose the sets $\cX_1,\ldots,\cX_r$ such that they contain $\cX_{\mathrm{SOS}}$ and show that the iterative projection behaves in a similar way as the projection to $\cX_{\mathrm{SOS}}$.
In this sense our algorithm is similar in spirit to iterated projective methods like the Bregman method (e.g., \cite{goldstein2009split}).

\paragraph{Dictionary learning}

In this basic unsupervised learning problem, the goal is to learn an unknown matrix $A\in \R^{d\times n}$ from i.i.d. samples $\super y 1 = A \super x 1, \ldots, \super y m = A \super x m$, where $\super x 1,\ldots,\super x m$ are i.i.d. samples from a distribution $\{x\}$ over sparse vectors in $\R^n$.
(Here, the algorithm has access only to the vectors $\super y 1,\ldots,\super y m$ but not to $\super x 1,\ldots,\super x m$.)

Dictionary learning, also known as sparse coding, is studied extensively in neuroscience~\cite{olshausen1997sparse}, machine learning~\cite{ArgyriouEP06,RanzatoBL2007}, and computer vision~\cite{EladA2006,MairalLBHP2008,YangWHY2008}.
Most algorithms for this problem used in practice do not come with strong provable guarantees.
In recent years, several algorithms with provable guarantees have been developed for this problem \cite{DBLP:conf/colt/AgarwalA0NT14, DBLP:journals/corr/AroraBGM14, DBLP:conf/colt/AroraGMM15, DBLP:conf/stoc/BarakKS15,DBLP:conf/focs/MaSS16, DBLP:conf/nips/HazanM16}.

For the case that the coordinates of the distribution $\{x\}$ are independent (and non-Gaussian) there is a well-known reduction\footnote{The reduction only requires $k$-wise independence where $k$ is the order to tensor the reduction produces.} of this problem to tensor decomposition where the components are the columns of $A$ and the error $E$ can be made inverse polynomially small by taking sufficiently many samples.
(In the case of independent coordinates, dictionary learning becomes a special case of independent component analysis / blind source separation, where this reduction originated.)
Variants of Jennrich's spectral tensor decomposition algorithm (e.g., \cite{DBLP:conf/stoc/BhaskaraCMV14,DBLP:conf/focs/MaSS16}) imply strong provable guarantees for dictionary learning in the case that $\{x\}$ has independent coordinates even in the ``overcomplete'' regime when $n\gg d$ (using a polynomial number of samples).
\Dnote{}

More challenging is the case that $\{x\}$ has non-independent coordinates, especially if those correlations are unknown.
We consider a model proposed in \cite{DBLP:conf/stoc/BarakKS15} (similar to a model in \cite{DBLP:journals/corr/AroraBGM14}):
We say that the distribution $\{x\}$ is \emph{$\tau$-nice} if
\begin{compactenum}
\item $\E x_i^4=1$ for all $i\in [n]$,
\item $\E x_i^2 x_j^2 \le \tau$ for all $i\neq j\in[n]$,
\item $\E x_i x_j x_k x_\ell=0$ unless $x_i x_j x_k x_\ell$ is a square.
\end{compactenum}
The conditions allow for significant correlations in the support set of the vector $x$.
For example, we can obtain a $\tau$-nice distribution $\{x\}$ by starting from any distribution over subsets $S\subseteq [n]$ such that $\Pr\{ i\in S\}=p$ for all $i\in [n]$ and $\Pr\{j\in S\mid i\in S \}\le \tau$ for all $i\neq j$ and choosing $x$ of the form $x_i =  p^{-1/4} \cdot \sigma_i$ if $i\in S$ and $x_i=0$ if $i\not\in S$, where $\sigma_1,\ldots,\sigma_n$ are independent random signs.

An extension of our aforementioned algorithm for orthogonal tensor decomposition with spectral norm error allows us to learn orthonormal dictionaries from $\tau$-nice distributions.
To the best of our knowledge, the only previous algorithms to provably solve this problem use sum-of-squares relaxations \cite{DBLP:conf/stoc/BarakKS15,DBLP:conf/focs/MaSS16}, which have large polynomial running time.
Our algorithm recovers a $0.99$ fraction of the columns of $A$ up to error $\tau$ from $\tO(n^3)$ samples, and runs in time $n^{3+O(\tau)}d^4$.
By a standard reduction, our algorithm also works for non-singular dictionaries and the running time increases by a factor polynomial in the condition number of $A$.
\Dnote{}

\subsection{Results}
\label{sec:results}

\paragraph{Tensor decomposition}
For tensor decomposition, we give an algorithm with close to linear running time that recovers the rank-1 components of a tensor with orthonormal components, so long as the spectral norm of the square unfoldings of the error tensor is small.
\begin{theorem}[Tensor decomposition with spectral norm error]
  There exists a randomized spectral algorithm with the following guarantees:
  Given a tensor $\bT \in (\R^d)^{\tensor 4}$ of the form $\bT = \sum_{i=1}^n a_i^{\tensor 4} + \bE$ such that $a_1,\ldots,a_n \in \R^d$ are orthonormal and $\bE$ has spectral norm at most $\e$ as a $d^2$-by-$d^2$ matrix,
    the algorithm can recover one of the components with $\ell_2$-error $O(\e)$ in time $\tO(d^{2+\omega + O(\eps)})$ with high probability, and recover $0.99n$ of the components with $\ell_2$-error $O(\e)$ in time $\tO(d^{2+\omega + O(\e)})$ with high probability.

    Furthermore, if $\epsilon \le O(\log\log n/\log^3 n)$, the algorithm can recover all components up to error $O(\epsilon)$ in time $\tO(d^{2+\omega} + nd^4)$ with high probability.
  Here, $\omega\le 2.373$ is the matrix multiplication exponent.
\end{theorem}

The orthogonality condition may at first seem restrictive, but for most applications it is possible to take a tensor with linearly independent components and transform it to a tensor with orthogonal components, as we will do for our dictionary learning result below.
Furthermore, our algorithm works as-is if the components are sufficiently close to orthonormal:

\begin{corollary}
    If we have that $a_1,\ldots,a_n$ are only approximately orthonormal in the sense that the $a_i$ are independent and $\left\|\sum_i a_ia_i^\top -\Id_S\right\| \le \eta$, where $\Id_S$ is the identity in the subspace spanned by the $a_i$, then we can recover $b_i$ so that $\iprod{a_i,b_i}^2 \ge 1 - O(\sqrt{\eta})$ with the same algorithm and runtime guarantees.
\end{corollary}

These robustness guarantees are comparable to those of the sum-of-squares-based algorithms in \cite{DBLP:conf/stoc/BarakKS15,DBLP:conf/focs/MaSS16} for the undercomplete case, which are the best known.
Meanwhile, the sum-of-squares based algorithms require solving large semidefinite programs, while the running time of our algorithms is close to linear in the size of the input, and our algorithms are composed of simple matrix-vector multiplications.

On the other hand, our algorithms fail to work in the overcomplete case, when the rank grows above $n$, and the components are no longer linearly independent.
One interesting open question is whether the techniques used in this paper can be extended to the overcomplete case.

\paragraph{Dictionary learning}

Using our tensor decomposition algorithm as a primitive, we give an algorithm for dictionary learning when the sample distribution is $\tau$-nice.

\begin{theorem}[Dictionary learning]\label{thm:dict-informal}
    Suppose that $A \in \R^{d\times n}$ is a dictionary with orthonormal columns, and that we are given random independent samples of the form $y = Ax$ for $x \sim \cD$.
    Suppose furthermore that $\cD$ is $\tau$-nice, as defined above, for $\tau < c^*$ for some universal constant $c^*$.

    Then there is a randomized spectral algorithm that recovers orthonormal vectors $b_1,\ldots,b_k \in \R^d$ for $k \ge 0.99 n$ with $\iprod{b_i,a_i}^2 \ge (1 - O(\tau))$, and with high probability requires $m = \tO(n^3)$ samples and time $\tO(d^{2+\omega} + n^{1+O(\tau)}d^4 + md^4)$.
\end{theorem}

The total runtime is thus $\tO(n^3 d^4)$---in the theorem statement, we write it in terms of the number of samples $m$ in order to separate the time spent processing the samples from the learning phase.
We note that the sample complexity bound that we have, $m = \tO(n^3)$, may very well be sub-optimal; we suspect that $m = \tO(n^2)$ is closer to the truth, which would yield a better runtime.

We are also able to apply standard whitening operations (as in e.g. \cite{DBLP:conf/colt/AnandkumarGHK13}) to extend our algorithm to dictionaries with linearly-independent, but non-orthonormal, columns, at the cost of polynomially many additional samples.

\begin{corollary}
    If $A \in \R^{d\times n}$ is a dictionary with linearly independent columns, then there is a randomized spectral algorithm that recovers the columns of $A$ with guarantees similar to \pref{thm:dict-informal} given $\tO(n^2 \cdot f(\mu))$ additional samples, where $\mu=\lambda_{\max}(AA^\top)/\lambda_{\min}(AA^\top)$ is the condition number of the covariance matrix, and $f$ is a polynomial function.
\end{corollary}

To our knowledge, our algorithms are the only remotely efficient dictionary learning algorithms with provable guarantees that permit $\tau$-nice distributions in which the coordinates of $x$ may be correlated by constant factors, the only other ones being the sum-of-squares semidefinite programming based algorithms of \cite{DBLP:conf/stoc/BarakKS15,DBLP:conf/focs/MaSS16}.
\section{Preliminaries}
\label{sec:preliminaries}

Throughout the rest of this paper, we will denote tensors by boldface letters such as $\bT$, matrices by capital letters $M$, and vectors by lowercase letters $v$, when the distinction is helpful.
We will use $A^{\tensor k}$/$u^{\tensor k}$ to denote the $k$th Kronecker power of a matrix/vector with itself.
To enhance legibility, for $u \in \R^d$ we will at times abuse notation and use $u^{\tensor 4}$ to denote the order-$4$ tensor $u\tensor u \tensor u\tensor u$, the $d^2 \times d^2$ matrix $(u^{\tensor 2})(u^{\tensor 2})^\top$, and the dimension $d^4$ vector $u^{\tensor 4}$---we hope the meaning will be clear from context.

For a tensor $\bT\in (\R^d)^{\otimes 4}$ and a partition of the modes $\{1,2,3,4\}$ into two ordered sets $A$ and $B$, we let $T_{A,B}$ denote the reshaping of $\bT$ as $d^{\card A}$-by-$d^{\card B}$ matrix, where the modes in $A$ are used to index rows and the modes in $B$ are used to index columns.
For example, $T_{\{1,2\}, \{3,4\}}$ is a $d^2$-by-$d^2$ matrix such that the entry at row $(i,j)$ and column $(k,\ell)$ contains the entry $T_{i,j,k,\ell}$ of $T$.
We remark that the order used to specify the modes matters---for example, for the rank-1 tensor $\bT = a \tensor b \tensor a \tensor b$, we have that $T_{\{1,2\}\{3,4\}} = (a \tensor b)(a\tensor b)^\top$ is a symmetric matrix, while $T_{\{2,1\}\{3,4\}} = (b \tensor a)(a\tensor b)^\top$ is not.
We use $\norm{T_{A,B}}$ to denote the spectral norm (largest singular value) of the matrix $T_{A,B}$.

We will also make frequent use of the following lemma, which states that the distance between two points cannot increase when both are projected onto a closed, convex set.
\begin{lemma*}
    Let $\cC\subset \R^n$ be a closed convex set, and let $\Pi:\R^n \to \cC$ be the projection operator onto $\cC$ in terms of norm $\|\cdot\|_2$, i.e. $\Pi(x) \defeq \argmin_{c \in \cC} \|x-c\|_2$.
    Then for any $x,y \in \R^n$,
    \[
	\|x-y\|_2 \ge \|\Pi(x) - \Pi(y)\|_2.
    \]
\end{lemma*}
This lemma is well-known (see e.g. \cite{Rock76}), but we will prove it for completeness in \pref{app:tools}.
\section{Techniques}
\label{sec:techniques}

In this section we give a high-level overview of the algorithms in our paper, and of their analyses.
We begin with the tensor decomposition algorithm, after which we'll explain the (non-trivial) extension to the dictionary learning application.
At the very end, we will discuss the relationship between our algorithms and sum-of-squares relaxations.

Suppose have a tensor $\bT \in (\R^d)^{\tensor 4}$, and that $\bT = \bS + \bE$ where the \emph{signal} $\bS$ is a low-rank tensor with orthonormal components, $\bS = \sum_{i\in[n]} a_i^{\tensor 4}$, and the \emph{noise} $\bE$ is an arbitrary tensor of noise with the restriction that for any reshaping of $\bE$ into a square matrix $E$, $\|E\| \le \epsilon$.
Our goal is to (approximately) recover the rank-1 components, $a_1,\ldots,a_n \in \R^d$, up to signs.

\paragraph{Failure of Jennrich's algorithm}

To motivate the algorithm and analysis, it first makes sense to consider the case when the noise component $\bE = 0$.
In this case, we can run Jennrich's algorithm: if we choose a $d^2$-dimensional random vector $g\sim \cN(0,\Id)$, we can compute the contraction
\[
    M_g \defeq \sum_{i,j=1}^n g_{ij} T_{ij} = \sum_{i=1}^n \iprod{g,a_i^{\tensor 2}} \cdot a_i a_i^{\top},
\]
where $T_{ij}$ is the $i,j$th $d \times d$ matrix slice of the tensor $\bT$.
Since the $a_i$ are orthogonal, the coefficients $\iprod{g,a_i^{\tensor 2}}$ are independent, and so we find ourselves in an ideal situation---$M_g$ is a sum of the orthogonal components we want to recover with independent Gaussian coefficients.
A simple eigendecomposition will recover all of the $a_i$.

On the other hand, when we have a nonzero noise tensor $\bE$, a random contraction along modes $\{1,2\}$ results in the matrix
\[
    M_g = \sum_{i=1}^n \iprod{g, a_i^{\otimes 2}} a_i a_i^\top + \sum_{i,j=1}^n g_{ij}\cdot E_{ij},
\]
where the $E_{ij}$ are $d\times d$ slices of the tensor $\bE$.
The last term, composed of the error, complicates things.
Standard facts about Gaussian matrix series assert that the spectral norm of the error term behaves like $\|E_{\{1,2,3\}\{4\}}\|$, the spectral norm of a $d^3 \times d$ reshaping of $E$, whereas we only have control over square reshapings such as $\|E_{\{1,2\}\{3,4\}}\|$.\footnote{In fact, this observation was crucial in the analysis of \cite{DBLP:conf/focs/MaSS16}---in that work, semidefinite programming constraints are used to control the spectral norm of the rectangular reshapings.}
These can be off by polynomial factors.
If the Frobenius norm of $\bE$ is $\|\bE\|_F^2 \approx \epsilon^2 d^2$, which is the magnitude one would expect from a tensor whose square reshapings are full-rank matrices with spectral norm $\epsilon$, then we have that necessarily
\[
    \|E_{\{1,2,3\}\{4\}}\|^2
    \ge \frac{\|\bE\|_F^2}{\rank\left(E_{\{1,2,3\}\{4\}}\right)}
    \ge\epsilon^2 d,
\]
since there are at most $d$ nonzero singular values of rectangular reshapings of $\bE$.
In this case, unless $\epsilon \ll 1/\sqrt{d}$, the components $a_ia_i^\top$ are completely drowned out by the contribution of the noise, and so the robustness guarantees leave something to be desired.

\paragraph{Basic idea}

The above suggests that, as long as we allow the error $\bE$ to have large Frobenius norm, an approach based on random contraction will not succeed.
Our basic idea is to take $\bT$, whose error has small spectral norm, and transform it into a tensor $\bT'$ whose error has small Frobenius norm.

Because we do not know the decomposition of $\bT$, we cannot access the error $\bE$ directly.
However, we do know that for any $d^2 \times d^2$ reshaping $T$ of $\bT$,
\[
T = S + E,
\]
where $S = \sum_{i=1}^n a_i^{\tensor 4}$, and $\|E\| \le \epsilon$.
The rank of $S$ is $n \le d$, and all eigenvalues of $S$ are $1$.
Thus, if we perform the operation
\[
    T^{>\eps} = (T-\eps\Id)_{+},
\]
where $(\cdot)_{+}$ denotes projection to the cone of positive semidefinite matrices, we expect that the signal term $S$ will survive, while the noise term $E$ will be dampened.
More formally, we know that $T$ has $n$ eigenvalues of magnitude $1 \pm \epsilon$, and $d^2 - n$ eigenvalues of magnitude at most $\epsilon$, and therefore $\rank(T^{>\eps}) \le n$.
Also by definition, $\|T - T^{>\eps}\| \le \epsilon$.
Therefore, we have that
\[
    S + E = T = T^{>\eps} + E'
\]
with $\|E'\| \le \epsilon$, and thus
\[
    \|T^{>\eps} - S\| = \|E - E'\|.
\]
Since $S,T^{>\eps}$ are both of rank at most $n$, and $E',E$ have spectral norm bounded by $\eps$, we have that
\[
    \|T^{>\eps}-S\|^2_F \le \left(\rank(S) + \rank(T^{>\eps})\right) \cdot \left(\|E\| + \|E'\|\right)^2\le 2n \cdot 4\epsilon^2.
\]
So, the Frobenius norm is no longer an impassable obstacle to the random contraction approach---using our upper bound on the Frobenius norm of our new error $\tilde{E} \defeq T^{>\eps} - S$, we have that the average squared singular value of $\tilde{E}_{\{1,2,3\}\{4\}}$ will be
\[
    \sigma^2_{avg}(\tilde{E}_{\{1,2,3\}\{4\}}) = \frac{\|\tilde{E}\|_F^2}{d} = O\left(\frac{n}{d}\epsilon^2\right).
\]
So while $\tilde{E}_{\{1,2,3\}\{4\}}$ may have large singular values, by Markov's inequality it cannot have too many singular values larger than $O(\epsilon)$.

Finally, to eliminate these large singular values, we will project $\bT^{>\epsilon}$ into the set of matrices whose rectangular reshapings have singular values at most $1$---because $S$ is a member of this convex set, the projection can only decrease the Frobenius norm.
After this, we will apply the random contraction algorithm, as originally suggested.

\paragraph{Variance of Gaussian matrix series}

Recall that we wanted to sample a random $d^2$-dimensional Gaussian vector $g$, and the perform the contraction
\[
    M_g \defeq \sum_{i,j=1}^d g_{ij} T^{>\eps}_{ij} = \sum_{i=1}^d \iprod{g,a_i^{\tensor 2}} \cdot a_i a_i^{\top} + \sum_{ij} g_{ij}\cdot \tilde{E}_{ij}.
\]
The error term on the right is a matrix Gaussian series.
The following lemma describes the behavior of the spectra of matrix Gaussian series:
\begin{lemma*}[See e.g. \cite{DBLP:journals/focm/Tropp12}]
    Let $g\sim \cN(0,\Id)$, and let $A_1,\ldots,A_k$ be $n\times m$ real matrices.
    Define $\sigma^2 = \max\left\{\left\|\sum_i A_i A_i^\top \right\|,\left\|\sum_i A_i^\top A_i \right\|\right\}$.
    Then
    \[
	\Pr\left( \left\|\sum_{i=1}^k g_i A_i \right\| \ge t \right) \le (n+m)\cdot \exp\left(-\frac{t^2}{2\sigma^2}\right)\mper
    \]
\end{lemma*}

For us, this means that we must have bounds on the spectral norm of both $\sum_{ij} \tilde{E}_{ij}\tilde{E}_{ij}^\top$ and $\sum_{ij}\tilde E_{ij}^\top \tilde E_{ij}$.
This means that if we have performed the contraction along modes $1$ and $2$, so that the index $i$ comes from mode $1$ and the index $j$ comes from mode $2$, then it is not hard to verify that $\|\sum_{ij} \tilde E_{ij}\tilde E_{ij}^\top\| = \|\tilde E_{\{1,2,3\}\{4\}}\|^2$, and $\|\sum_{ij} \tilde E_{ij}^\top\tilde E_{ij}\| = \|\tilde E_{\{1,2,4\}\{3\}}\|^2$.
So, we must control the maximum singular values of two different rectangular reshapings of $\tilde E$ simultaneously.

It turns out that for us it suffices to perform two projections in sequence---we first reshape $\bT^{>\eps}$ to the matrix $\bT^{>\eps}_{\{123\}\{4\}}$,project it to the set of matrices with singular values at most $1$, and then reshape the result along modes $\{124\}\{3\}$, and project to the same set again.
As mentioned before, because projection to a convex set containing $S$ cannot increase the distance from $S$, the Frobenius norm of the new error can only decrease.
What is less obvious is that performing the second projection will not destroy the property that the reshaping along modes $\{123\}\{4\}$ has spectral norm at most $1$.
By showing that each projection corresponds to either left- or right- multiplication of $\bT^{>\eps}_{\{123\}\{4\}}$ and $\bT^{>\eps}_{\{124\}\{3\}}$ by matrices of spectral norm at most $1$, we are able to show that the second projection does not create large singular values for the first flattening, and so two projections are indeed enough.
Call the resulting tensor $(\bT^{>\eps})^{\le 1}$.

Now, if we perform a random contraction in the modes $1,2$, we will have
\[
    M_g = \sum_{ij} g_{ij}(\bT^{>\eps})^{\le 1}_{ij} = \sum_i \iprod{a_i^{\tensor 2},g} a_i a_i^{\top} + \hat E_g,
\]
where the spectral norm of $\|\hat E_g\| \le \sqrt{\log n}$ with good probability.
So, ignoring for the moment dependencies between $\iprod{a_i^{\tensor 2},g}$ and $E_g$, $\max_i|\iprod{g,a_i^{\tensor 2}}| > 1.1\cdot \|E_g\|$ with probability at least $n^{-.2}$, which will give $a_i$ correlation $0.9$ with the top eigenvector of $M_g$ with good probability.

\paragraph{Improving accuracy of components}

The algorithm described thus far will recover components $b_i$ that are $0.9$-correlated with the $a_i$, in the sense that $\iprod{b_i,a_i}^2 \ge 0.9$.
To boost the accuracy of the recovered components, we use a simple method which resembles a single step of tensor power iteration.

We'll use the closeness of our original tensor $\bT$ to $\sum_i a_i^{\tensor 4}$ in spectral norm.
We let $T$ be a $d^2 \times d^2$ flattening of $\bT$, and compute the vector
\[
    v = T(b_i \tensor b_i) = 0.9 \cdot a_i\tensor a_i + \sum_{j\neq i} \iprod{b_i,a_j}^2 a_j \tensor a_j + E(b_i \tensor b_i).
\]
Now, when the vector $v$ is reshaped to a $d \times d$ matrix $V$, the term $E(b_i\tensor b_i)$ is a matrix of Frobenius norm (and thus spectral norm) at most $\epsilon$.
By the orthonormality of the $a_j$, the sum of the coefficients in the second term is at most 0.1, and so $a_i$ is $\epsilon$-close to the top eigenvector of $V$.

\paragraph{Recovering every component}
Because the Frobenius norm of the error in  $(\bT^{>\eps})^{\le 1}$ is $O(\epsilon \sqrt{n})$, there may be a small fraction of the components $a_i^{\tensor 4}$ that are ``canceled out'' by the error---for instance, we can imagine that the error term is $\hat \bE = - \sum_{i=1}^{\epsilon^2 n} a_i^{\tensor 4}$.
So while only a constant fraction of the components $a_i^{\tensor 4}$ can be more than $\epsilon$-correlated with the error, we may still be unable to recover some fixed $\epsilon^2$-fraction of the $a_i$ via random contractions.

To recover all components, we must subtract the components that we have found already and run the algorithm iteratively---if we have found $m = 0.99n$ components, then if we could perfectly subtract them from $\bT$, we would end up with an even lower-rank signal tensor, and thus be able to make progress by truncating all but $0.1n$ eigenvalues in the first step.

The challenge is that we have recovered $b_1,\ldots,b_m$ that are only $(1-\epsilon)$-correlated with the $a_i$, and so naively subtracting $\bT - \sum_{i=1}^m b_i^{\tensor 4}$ can result in a Frobenius and spectral norm error of magnitude $\epsilon \sqrt{m}$---thus the total error is still proportional to $\sqrt{n}$ rather than $\sqrt{0.1n}$.

In order to apply our algorithm recursively, we first orthogonalize the components we have found $b_1,\ldots,b_m$ to obtain new components $\tilde b_1,\ldots, \tilde b_m$.
Because the $b_i$ are close to the truly orthonormal $a_i$, the orthogonalization step cannot push too many of the  $\tilde{b_i}$ more than $O(\epsilon)$-far from the $b_i$---in fact, letting $B$ be the matrix whose columns are the $b_i$, and letting $A$ be the matrix whose columns are the corresponding $a_i$, we use that $\|A - B\|_F \le O(\epsilon)\sqrt{m}$, and that the matrix $\tilde B$ with columns $\tilde{b}_i$ is closer to $B$ than $A$.
We keep only the $\tilde{b}_i$ for which $\iprod{\tilde{b}_i^{\tensor 4},\bT} \ge 1 - O(\epsilon)$, and we argue that there must be at least $0.9m$ such $\tilde{b}_i$.
Let $K \subset [m]$ be the set of indices for which this occurred.

Now, given that the two sets of orthogonal vectors $\{\tilde{b}_i\}_{i\in K}$ and $\{a_i\}_{i\in K}$ are all $O(\epsilon)$ close, we are able to prove that
\[
    \left\|\sum_{i\in K} \tilde{b}_i^{\tensor 4} - a_i^{\tensor 4} \right\|\le O(\sqrt{\epsilon}).
\]
So subtracting the $\tilde b_i^{\tensor 4}$ will not introduce a large spectral norm!
Since we will only need to perform this recursion $O(\log n)$ times, allowing for some leeway in $\epsilon$ (by requiring $\sqrt{\epsilon}\log n = o(1)$), we are able to recover all of the components.

\paragraph{Dictionary learning}

In the dictionary learning problem, there is an unknown dictionary, $A \in \R^{d\times n}$, and we receive independent samples of the form $y = Ax$ for $x \sim \cD$ for some distribution $\cD$ over $\R^n$.
The goal is, given access only to the samples $y$, recover $A$.

We can use our tensor decomposition algorithm to learn the dictionary $A$, as long as the columns of $A$ are linearly independent.
For the sake of this overview, assume instead that the columns of $A$ are orthonormal.
Then given samples $y^{(1)},\ldots,y^{(m)}$ for $m = \poly(n)$, we can compute the $4$th moment tensor to accuracy $\epsilon$ in the spectral norm,
\[
    \frac{1}{m}\sum_{j=1}^m (y^{(j)})^{\tensor 4}
    \approx \E_{x\sim \cD}\left[(Ax)^{\tensor 4}\right]\mper
\]
If the right-hand side were close to $\sum_i a_i^{\tensor 4}$ in spectral norm, we would be done.
However, for almost any distribution $\cD$ which is supported on $x$ with more than one nonzero coordinate, this is not the case.
If we assume that $\E[x_ix_jx_kx_\ell] = 0$ unless $x_ix_jx_kx_\ell$ is a square, then we can calculate that any square reshaping of this tensor will have the form
\begin{align*}
    \E_{x\sim \cD}\left[(Ax)^{\tensor 4}\right]
    &= \sum_i \E[x_i^4] \cdot a_i^{\tensor 4} + \sum_{i\neq j} \E[x_i^2 x_j^2] \cdot ( a_ia_j^\top \tensor a_ia_j^\top + a_ia_i^\top \tensor a_ja_j^\top + a_ia_j^\top \tensor a_ja_i^\top)\mper
\end{align*}
If $\E[x_i^4] = \E[x_j^4]$ for all $i,j$, then the first term on the right is exactly the $4$th order tensor that we want.
The second term on the right can be further split into three distinct matrices, one for each configuration of the $a_i,a_j$.
The $a_ia_i^\top \tensor a_ja_j^\top$ term and the $a_ia_j^\top \tensor a_ja_i^\top$ terms can be shown to have spectral norm at most $\max_{i\neq j}\E[x_i^2x_j^2]$, and so as long as we require that $\max_{i\neq j}\E[x_i^2x_j^2]\ll \epsilon \E[x_k^4]$, these terms have spectral norm within the allowance of our tensor decomposition algorithm.

The issue is with the $a_ia_j^\top \tensor a_ia_j^\top$ term.
This term factors into $(a_i\tensor a_i)(a_j\tensor a_j)^\top$, and because of this the entire sum $\sum_{i\neq j} \E[x_i^2x_j^2] \cdot a_ia_j^\top \tensor a_ia_j^\top$ has rank at most $n\ll d^2$, but Frobenius norm as large as $\max_{i\neq j} \E[x_i^2 x_j^2] \cdot n$.
If we require that the coordinates of $x\sim \cD$ are independent, then we can see that this is actually close to a spurious rank-$1$ component, which can be easily removed without altering the signal term too much.\footnote{As was done in \cite{DBLP:journals/corr/HopkinsSSS15}, for example, albeit in a slightly different context.}
However, if we wish to let the coordinates of $x$ exhibit correlations, we have very little information about the spectrum of this term.

In the sum-of-squares relaxation, this issue is overcome easily: because of the symmetries required of the SDP solution matrix $X$, $\iprod{X,a_ia_j^\top \tensor a_ia_j^\top} = \iprod{X, a_ia_i^\top \tensor a_ja_j^\top}$, so by linearity this low-rank error term cannot influence the objective function any more than the $a_ia_i^\top \tensor a_ja_j^\top$ term.

Inspired by this sum-of-squares analysis, we remove these unwanted directions as follows.
Given the scaled moment matrix $M = \frac{1}{\E[x_1^4]}\E_{x\sim \cD}\left[(Ax)^{\tensor 4}\right]$ (where the scaling serves to make the coefficients of the signal $1$), we truncate the small eigenvalues:
\[
    M^{>\eps} = (M_{\{1,2\}\{3,4\}} - \eps \Id)_+.
\]
This removes the spectrum in the direction of the ``nice'' error terms, corresponding to $a_ia_i^\top \tensor a_ja_j^\top$ and $a_ia_j^\top\tensor a_ja_i^\top$.
The fact that the rest of the matrix is low-rank means that we can apply an analysis similar to the analysis in the first step of our tensor decomposition to argue that $\|M^{>\eps}-\sum_i a_i^{\tensor 4}\|_F \le\|M- \sum_i a_i^{\tensor 4} \|_F$.

Now, we re-shape $M^{>\eps}$, so that if initially we had the flattening $\bM \to M_{\{1,2\}\{3,4\}}$, we look at the flattening $M^{>\eps}_{\{1,3\}\{2,4\}}$.
In this flattening, the term $a_ia_j^\top\tensor a_ia_j^\top$ from $M^{>\eps}$ is transformed to $a_ia_i^\top \tensor a_ja_j^\top$, and so the problematic error term from the original flattening has spectral norm $\epsilon$ in this flattening!
Applying the projection
\[
    (M^{>\eps}_{\{1,3\}\{2,4\}} - \eps\Id)_+
\]
eliminates the problematic term, and brings us again closer to the target matrix $\sum_i a_i^{\tensor 4}$.
Thus, we end up with a tensor that is close to $\sum_i a_i^{\tensor 4}$ in Frobenius norm, and we can apply our tensor decomposition algorithm.

\paragraph{Connection to sum-of-squares algorithms}
We take a moment to draw parallels between our algorithm and tensor decomposition algorithms in the sum-of-squares hierarchy.
In noisy orthogonal tensor decomposition, we want to solve the non-convex program
\[
    \argmax \iprod{X, \bT} \quad \text{s.t.}\quad \left\{X\in\R^{d^2\times d^2},\ X \succeq 0,\ \|X\|= 1,\ \rank(X) = n, \ X \in \Span\{u^{\tensor 4} : u \in \R^d\}\right\}\mper
\]
The intended solution of this program is $X = \bS$ (after which we can run Jennrich's algorithm to recover individual components).
At first it may not be obvious that the maximizer of the above program is close to $\bS$, but for any unit vector $x \in \R^d$,
\[
    \iprod{x^{\tensor 4}, \bT} = \sum_{i=1}^n \iprod{x,a_i}^4 + (x\tensor x)^\top E (x \tensor x).
\]
The error term is at most $\eps$ by our bound on $\|E\|$, and the first term is $\|x^\top A\|_4^4$, where $A$ is the matrix whose columns are the $a_i$.
Since by the orthonormality of the $a_i$, $\|x^\top A\|_2 \le 1$, and the $\ell_4$ norm is maximized relative to the $\ell_2$ for vectors supported on a single coordinate, the $x$ that maximize this must be $\epsilon$-close to one of the $a_i$.
In conjunction with the $\|X\|\le 1$ constraint and the $\rank(X) = n$ constraint, we have that $X := \bS$ is the maximizer.

The (somewhat simplified) corresponding sum-of-squares relaxation is the semidefinite program
\[
    \max \iprod{X,\bT} \quad \text{s.t.} \quad \left\{X \in \R^{d^2\times d^2},\ X \succeq 0,\ \|X\| \le 1,\ \|X\|^2_F = n, \ X_{ijk\ell} = X_{\pi(ijk\ell)}\ \forall \pi \in \cS_{4} \right\},\footnote{
	The program is actually over matrices indexed by all subsets of $d$ of size at most $2$, $\binom{d}{\le 2}$, but for simplicity in this description we ignore this (and the interaction with the SOS variables corresponding to lower-degree monomials or lower-order moments).}
\]
The constraints $\|X\|_F^2 = n$ and $X_{ijk\ell} = X_{\pi(ijk\ell)}$ together are a relaxation of the constraint that $X$ be a rank-$n$ matrix in the symmetric subspace $\Span\{u^{\tensor 4}\}$.
Further, the constraint $\|X\| \le 1$ is enforced in every rectangular $d \times d^3$ reshaping of $X$ (this consequence of the SOS constraints is crucially used in \cite{DBLP:conf/focs/MaSS16}).

To solve this semidefinite program, one should project $\bT$ into the intersection of all of the convex feasible regions of the constraints.
However, projecting to the intersection is an expensive operation in terms of runtime.
Instead, we choose a subset of these constraints, and project $\bT$ into the set of points satisfying each constraint sequentially, rather than simultaneously.
These are not equivalent projection operations, but because we select our operations carefully, we are able to show that our $\bT$ is close to the SDP optimum in a sense that is sufficient for successfully running Jennrich's algorithm.

\medskip

In the first step of our algorithm, we change the objective function from $\iprod{X,\bT}$ to $\iprod{X,\bT - \eps\Id}$.
In the sum-of-squares SDP, this does not change the objective value dramatically---because the original objective value is at least $n$, and because the Frobenius norm constraint $\|X\|_F^2 = n$ constraint in conjunction with the sum-of-squares constraints implies that $\iprod{X,\Id} = \eps n$.
Therefore this perturbation cannot decrease the objective by more than a multiplicative factor of $\epsilon$.
Then, we project a square reshaping of the objective to the PSD cone, $(T - \eps\Id)_+$---this corresponds to the constraint that $X \succeq 0$.\footnote{For a proof that truncating the negative eigenvalues of a matrix is equivalent to projection to the PSD cone in Frobenius norm, see \pref{fact:projcont}.}
Finally, we project first to the set of matrices that have spectral norm at most $1$ for one rectangular reshaping, then repeat for another rectangular reshaping.
So after perturbing the objective very slightly, then choosing three of the convex constraints to project to in sequence, we end up with an object that approximates the maximizer of the SDP in a sense that is sufficient for our purposes.

Our dictionary learning pre-processing can be interpreted similarly.
We first perturb the objective function by $\epsilon \Id$, and project to the PSD cone.
Then, in reshaping the tensor again, we choose another point that has the same projection onto any point in the feasible region (by moving along an equivalence class in the symmetry constraint).
Finally, we perturb the objective by $\epsilon \Id$ again, and again project to the PSD cone.
\section{Decomposing orthogonal $4$-tensors}\label{sec:orthog}
Recall our setting: we are given a 4-tensor $\bT \in (\R^{d})^{\otimes 4}$ of the form $\bT = \bS + \bE$ where $\bE$ is a noise tensor and $\bS = \sum_{i=1}^n a_i^{\otimes 4}$ for orthonormal vectors $a_1,\ldots,a_n$.
(We address the more general case of nearly orthonormal vectors in \pref{sec:northo}.).

First, we have a pre-processing step, in which we go from a tensor with low spectral norm error to a tensor with low Frobenius norm error.\footnote{We are approximating $T = S+E$ with $\rank(S) = n \ll d^2$, and $\|E\|\le \epsilon$, so for example if $E = \eps \Id$ then we may have $\|T - S\|_F^2 = \epsilon^2d^2$, which is too large for us.}
\begin{algorithm}[Preprocessing: spectral-to-Frobenius norm]\label{alg:preproc}{\color{white}.}
\\
    {\em Input:} A tensor $\bT \in (\R^d)^{\tensor 4}$, and an error parameter $\epsilon$.
    \begin{compactenum}
    \item Reshape $\bT$ to the $d^2 \times d^2$ matrix $T \defeq T_{\{1,2\}\{3,4\}}$.
    \item Truncate to 0 all eigenvalues of  that have magnitude  less than $\epsilon$:
	\[
	    T^{>\eps}\defeq (T - \eps\Id)_+,
	\]
	where we have used $(M)_+$ to denote projection to the PSD cone.
    \end{compactenum}
    {\em Output:} The tensor $T^{>\eps}$.
\end{algorithm}

\begin{lemma}\label{lem:step1}
    Suppose that for some square reshaping $E$ of $\bE$ (without loss of generality along modes $\{1,2\},\{3,4\}$), $\|E\| \le \epsilon$.
    Say we are given access to $\bT = \bS + \bE$, and we produce the matrix $T^{>\eps} = (T - \epsilon\cdot \Id)_+$ as described in \pref{alg:preproc}.
    Then $\|T' - S\|_F \le 2\epsilon\sqrt{2n}$.
    This operation requires time $\tO(\min\{nd^4,d^{2+\omega}\})$.
\end{lemma}
\begin{proof}
    Because $\|E\| \le \epsilon$, $T = S + E$ has only $n$ eigenvalues of magnitude more than $\epsilon$.
    So $\rank(T^{>\eps})\le n$, and therefore $\rank(T^{>\eps} - S)\le 2n$.
    Furthermore, $S + E = T = T^{>\eps} + \tilde{E}$ for a matrix $\tilde{E}$ of spectral norm at most $\epsilon$.
    Therefore $\|T^{>\eps} - S\| \le 2\epsilon$, and $\|T^{>\eps} - S\|_F \le 2\epsilon\sqrt{2n}$.

    To compute $T^{>\epsilon}$, we can compute the top $n$ eigenvectors of $T$.
    Since $O(\epsilon)\cdot \lambda_n \ge\lambda_{n+1}$, where $\lambda_n,\lambda_{n+1}$ are the $n$th and $(n+1)$st eigenvalues, we can compute this in time $\tO(\min\{nd^4,d^{2+\omega}\})$ via subspace power iteration (see for example \cite{DBLP:conf/nips/HardtP14}).
\end{proof}

Now, we can run our main algorithm:
\begin{algorithm}\label{alg:orthog}{\color{white}.}\\
    {\bf Input:} A tensor $\bT \in (\R^d)^{\tensor 4}$.
    \begin{compactenum}
    \item Project $\bT$ to the set of tensors whose rectangular reshapings along modes $\{1,2,3\},\{4\}$ have spectral norm at most $1$ (obtaining a new tensor $\hat \bT$ with $\|\hat\bT_{\{1,2,3\}\{4\}}\|\le 1$).
    \item Project $\hat \bT$ to the set of tensors whose rectangular reshapings along modes $\{1,2,4\},\{3\}$ have spectral norm at most $1$, obtaining a new tensor $\bT^{\le 1} = \bS + \bE^{\le 1}$.
    \item Sample $g \sim \cN(0,\Id_{d^2})$, and compute the random flattening $M_g \defeq \sum_{j=1}^{d^2} g_j \bT^{\le 1}_j$.
    \end{compactenum}
    {\bf Output:}
    $u_L$ and $u_R$, the top left- and right- unit singular vectors of $M_g$.
\end{algorithm}

Under the appropriate conditions on $\bT$, with probability $\tO(n^{-\eps})$, \pref{alg:orthog} will output a vector that is $0.9$-correlated with $a_j^{\otimes 2}$ for some $j\in[n]$.

\begin{theorem}\label{thm:orthog}
    Suppose we are given a $4$-tensor $\bT \in (\R^d)^{\tensor 4}$, and $\bT = \sum_i a_i^{\otimes 4} + E$, where $a_1,\ldots,a_n \in \R^d$ are orthonormal vectors and $\|E\|_F \le \eta\sqrt{n}$.

    Then running \pref{alg:orthog} $\tO(n^{1+O(\eta)})$ times allows us to recover $m \ge 0.99n$ unit vectors $u_1,\ldots,u_{m} \in \R^{d}$ such that for each $i \in [m]$, there exists $j \in [n]$ such that
    \[
	\iprod{u_i,a_j}^2 \ge 0.99.
    \]
Recovering one component requires time $\tO(d^{2+\omega}n^{O(\eta)})$, and recovering $m$ components requires time $\tO(\max\{mn^{O(\eta)}d^4,d^{2+\omega}\})$.
\end{theorem}

We can then post-process the vectors to obtain a vector that has correlation $1-\epsilon$ with $a_j$---the details are given in \pref{sec:fullrecovery} below.

We will prove \pref{thm:orthog} momentarily, but first, we bring the reader's attention to a nontrivial technical issue left unanswered by \pref{thm:orthog}.
The issue is that we can only guarantee that \pref{alg:orthog} recovers $0.99 n$ of the vectors, and the set of recoverable vectors is invariant under the randomness of the algorithm.
That is, as a side effect of the error-reducing step 2, the $\bS$ part of $\bT$ may also be adversely affected.
For that reason, in $\tilde{O}(n)$ runs of \pref{alg:orthog}, we can only guarantee that recover a constant fraction of the components, and we must iteratively remove the components we recover in order to continue to make progress.
This removal must be handled delicately to ensure that the Frobenius norm of the error shrinks at each step, so the conditions of \pref{thm:orthog} continue to be met (for shrinking values of $n$).
The overall algorithm, which uses \pref{alg:orthog} as a subroutine, will be given in \pref{sec:fullrecovery}.

We will now prove the correctness of \pref{alg:orthog} step-by-step, tying details together at the end of this subsection.
First, we argue that in step 1, the truncation of the large eigenvalues cannot increase the Frobenius norm of the error.

\begin{lemma}\label{lem:step2}
    Suppose that we define $\bT^{\le 1}$ to be the result of projecting $\bT = \bS + \bE$ to the set of tensors whose rectangular reshapings along modes $\{1,2,3\},\{4\}$ have spectral norm at most $1$, then projecting the result to the set of tensors whose rectangular reshapings along modes $\{1,2,3\},\{4\}$ have spectral norm at most $1$.
    Then $\|\bT^{\le 1} - \bS\|_F \le \|E\|_F$, and
    \[ \|\bT^{\le 1}_{\{1,2,3\}\{4\}}\|\le 1, \quad \text{and} \quad
	\|\bT^{\le 1}_{\{1,2,4\}\{3\}}\|\le 1\mper
	\]
	This operation requires time $\tO(d^{2+\omega})$.
\end{lemma}
\begin{proof}
    To establish the first claim, we note that the tensor $\bT^{\le 1}$ was obtained by two projections of different rectangular reshapings of the matrix $S+E$ to the set of rectangular matrices with singular value at most $1$.
    This set is closed, convex, and contains $S$, and so the error can only decrease in Frobenius norm (see \pref{lem:proxproj} in \pref{app:tools} for a proof),
    \[
	\|\bT^{\le 1} - \bS\|_F \le \|\bT - \bS\|_F = \|\bE\|_F.
    \]

    It is not hard to see that each projection step can be accomplished by reshaping the tensor to the appropriate rectangular matrix, then truncating all singular values larger than $1$ to $1$.
Now, we establish the remaining claims.
    For convenience, define $\hat{T}^{\le 1}$ to be the matrix $(S+E)_{\{123\}\{4\}}$ after restricting singular values of magnitude $>1$ to $1$.
    Now, we reshape $\hat{T}^{\le 1}$ to a new matrix $B \defeq \hat{T}^{\le 1}_{\{1,2,4\}\{3\}}$, which has the $d^2$ blocks $B_1 = (\hat{T}^{\le 1})_1^{\top},\ldots,B_{d^2}=(\hat{T}^{\le 1})_{d^2}^{\top}$.
    Say that the singular value decomposition of $B$ is $B = U \Sigma V^{\top}$.
    Define $\widehat{\inv{\Sigma}}$ to be the diagonal matrix with entries equal to those of $\inv{\Sigma}$ when the value is $< 1$ and with ones elsewhere.
    When we truncate the large singular values of $B$ this is equivalent to multiplying by $P = U\widehat{\inv{\Sigma}}U^\top$.
    The result is the matrix $PB$, with blocks $PB_1 = P(\hat{T}^{\le 1})_1^{\top},\ldots,PB_{d^2} = P(\hat{T}^{\le 1})_{d^2}^{\top}$.
    By definition, the singular values of $PB = T^{\le 1}_{\{3\}\{1,2,4\}}$ are at most $1$.
    Also, $T^{\le 1}_{\{4\}\{123\}} = \hat{T}^{\le 1}(P\otimes \Id_{d^2})$, and by the submultiplicativity of the norm, $\|\hat{T}^{\le 1}(P\otimes \Id_{d^2})\| \le \|\hat{T}^{\le 1}\|\cdot \|P\otimes \Id_{d^2}\| \le 1$, and the first reshaping still has spectral norm at most 1.

    Finally, each reshaping step takes $O(d^4)$ time.
    Since we are only interested in truncating large singular values, it suffices for us to compute the SVD corresponding to singular values between $\sqrt{n}$ and $1$.
    This can be done via subspace power iteration, which here involves the multiplication of a $d^3 \times d$ matrix and a $d \times d$ matrix (with intermediate orthogonalization steps for the $d \times d$ matrix, see \cite{DBLP:conf/nips/HardtP14}), which requires time $\tO(d^{2+\omega})$, where $\omega$ is the matrix multiplication constant.
    Going forward the representation of the matrix will be as the original matrix, with the subtracted SVD corresponding to large singular values.
\end{proof}

We will need to argue that if the Frobenius norm of the error matrix is small, this is a sufficient condition under which we succeed.
For this, we will use the following two lemmas.
The first tells us that with probability $\tOmega(n^{-O(\eps)})$ over the choice of $g$, we will have for some $i\in[n]$ that
\[
    M_g = c\cdot a_i a_i^{\top} + N,
\]
where $|c| \ge \|N\|$ and furthermore $\|Na_i\|$ and $\|N^{\top}a_i\|$ are small:

\begin{lemma}\label{lem:algsuccess}
    Let $g \sim \cN(0,\Id_{d^2})$.
    Suppose that $\|T_{\{1,2,3\}\{4\}}\|,\|T_{\{1,2,3\}\{4\}}\| \le 1$, and that $\|T - \sum_i a_i^{\tensor 4}\|_F \le \epsilon \sqrt{n}$.
    Define the matrix $M_g$ to be the flattening of $\bT = \sum_{i} a_i^{\otimes 4} + \bE $ along $g$ in the modes $\{1\}$ and $\{2\}$, and let $|c|$ be the magnitude of $a_j$'s projection onto $M_g$, i.e.
    \[
	M_g
	~\defeq~ \sum_{j=1}^{d^2}\iprod{e_j,g}\cdot T_j
	~=~ c \cdot a_j a_j^{\top} + N\mper
    \]
Then for a $1-3\delta$ fraction of $j \in [n]$,
    \[
	\Pr_g\left[~|c| \ge (1+\beta)\|N\|,~ \|N^{\top}a_j\|,\|Na_j\| \le (\epsilon/\delta)(c+\sqrt{2}+o(1))  \right] = \tOmega\left(n^{-\left(\frac{1+\beta}{1-(1+\beta)\epsilon/\delta}\right)^2}\right)\mper
    \]
In particular, if $\delta = \Omega(1)$, $\beta = O(\epsilon)$, $\beta < 1$, then this probability is $\tOmega(n^{-(1+O(\epsilon))})$.
\end{lemma}
The proof consists primarily of the application of concentration inequalities, and we provide it below in \pref{sec:supporting}.

The second lemma states that if $M_g$ indeed has the form above, the top singular vectors of $M_g$ must be close to the component $a_i$.
\begin{lemma}\label{lem:topeig}
    Let $M_g$ be an $n \times n$ matrix, and $a_1 \in \R^n$, and suppose that
    \[
	M_g = c \cdot a_1 a_1^\top + N
    \]
with $|c| \ge (1+\beta)\|N\|$ for $\beta > 0$, and $\|N a_1\|, \|N^\top a_1\| \le \epsilon|c|$ so that the relationship $\frac{2\epsilon(1+\beta)}{\beta} <  0.01 $ holds.
    Then letting $u$ be a top singular vector of $M_g$, it follows that
    \[
	\iprod{u,a_1}^2 \ge 0.99.
    \]
\end{lemma}
The proof requires some careful calculations, but is not complicated, and we will prove it below in \pref{sec:supporting}.

Finally, we are ready to stitch these arguments together and prove that \pref{alg:orthog} works.
    \begin{proof}[Proof of \pref{thm:orthog}]
	After reshaping and truncating $\bT$ in step $1$ of the algorithm, by \pref{lem:step2} the matrix $T^{\le 1} = \sum_i a_i^{\otimes 4} + E$ has the properties that
	\[
	    \|T^{\le 1}_{\{1,2,3\}\{4\}}\|, \|T^{\le 1}_{\{1,2,4\}\{3\}}\|\le 1,
	\]
	and also that still $\|E\|_F \le \eta \sqrt{n}$.

	We can now apply \pref{lem:algsuccess} with $\delta = \frac{1}{300}$ and $\beta = 400 \eta/\delta = O(\eta)$ to conclude that for at least a $0.99$-fraction of the $i\in[n]$, with probability at least $\tO(n^{-1 - O(\eta)})$, we will have
	\[
	    M_g = c \cdot a_i a_i^{\top} + N,
	\]
	where $\|N\| \le (1+\beta)c$ and $\|Na_i\|,\|N^\top a_i\| \le 48\eta\cdot |c|$.
	Applying \pref{lem:topeig}, we have that either the left- or right- top unit singular vector $u$ of $M_g$ has correlation at least
	\[
	    \iprod{u, a_i}^2 \ge 0.99,
	\]
	as desired.

	For runtime, by our arguments in \pref{lem:step2} step 1 takes time $\tO(d^{2+\omega})$.
	After this, with either representation of our matrix $T^{\le 1}$ (whether we compute the full truncated SVD or have the original matrix minus the subtracted SVD), performing power iteration to find the top eigenvector with the flattening $M_g$ takes time $\tO(d^4)$, and finding a single component takes $\tO(n^{-O(\eta)})$ samples of random contractions.
	Since we can reuse $T^{\le 1}$ with new random contractions, the total runtime for recovering one component is $\tO(d^{2+\omega}n^{-O(\eta)})$, and by the independence of the runs recovering $m$ components requires $\tO(d^{2+\omega}) + mn^{-O(\eta)}\cdot \tO(d^4)$ time.
    \end{proof}

With the core of our algorithm in place, we now take care of the remaining technical issues: recovery precision, working with near-orthonormal vectors, and recovering the full set of component vectors.

\subsection{Postprocessing for closer vectors}\label{sec:postproc}

Because the precision of recovery will be important in not amplifying the error, we begin with our precision-amplifying postprocessing algorithm.

\begin{algorithm}[Postprocessing for error reduction]\label{alg:post}
    {\color{white}.}\\
    {\em Input:} A tensor $\bT \in (\R^d)^{\otimes 4}$, a vector $u \in \R^{d^2}$, and an error parameter $\epsilon \ge \|E_{\{12\}\{34\}}\|$.
    \begin{compactenum}
    \item Compute the matrix-vector product $ a \defeq T_{\{1,2\}\{3,4\}}(u\tensor u)$.
    \item Reshape $a\in \R^{d^2}$ to a $d \times d$ matrix $A$, and compute the top left- and right- singular vectors $v_L$ and $v_R$ of $A$.
    \end{compactenum}
    {\em Output:}
	If for one of $v \in \{v_L,v_R\}$, $(v^{\otimes 2})^{\top}T v^{\otimes 2} \ge (1-3\epsilon)^2 - \epsilon$, output $v$.
\end{algorithm}

    \begin{lemma}\label{lem:postproc}
	Suppose that $v$ is a unit vector with $\iprod{v,a_i}^2 \ge 0.99$, and $T = \sum_i a_i^{\tensor 4} + E$ for $\|E\| \le \epsilon$ and $a_1,\ldots,a_n$ orthonormal.
	Then if we let $A$ be the reshaping of $T(v\tensor v)$ to a $d\times d$ matrix, and if we let $u_L,u_R$ be the top left- and right- unit singular vectors of $M$, then
	\[
	    \iprod{u_L,a_i}^2 \ge 1-3\epsilon \quad \text{or}\quad  \iprod{u_R,a_i}^2 \ge 1-3\epsilon.
	\]
	In other words, \pref{alg:post} succeeds.
	Further, the time required is $\tO(d^4)$.
    \end{lemma}
    \begin{proof}[Proof of \pref{lem:postproc}]
	For convenience and without loss of generality, let $i := 1$, and let $\alpha\defeq 1 - \iprod{a_1,v}^2 \le 0.01$.
	Because $a_1,\ldots, a_n$ are orthonormal, we can write $v = \sum_j \iprod{a_j,v} \cdot a_j + w$, where $w \perp a_j$ for all $j \in [n]$.
	By assumption, $\iprod{a_1,v}^2 \ge 1-\alpha$, and therefore $\sum_{j>1} \iprod{v,a_j}^2 + \|w\|^2 \le \alpha$.
	Now,
	\begin{align*}
	    (T-E) (v \tensor v)
	    &= \sum_{j} a_j^{\tensor 2}(a_j^{\tensor 2})^\top \left(\sum_{j} \iprod{a_j,v} a_j \tensor\sum_{j} \iprod{a_j,v} a_j\right)\\
	    &= \sum_{j,k,\ell} a_j^{\tensor 2}\iprod{a_j,a_k}\iprod{a_j,a_\ell} \iprod{a_k,v} \iprod{a_\ell,v}
	    \intertext{by the orthonormality of the $a_i$,}
	    &= \sum_{j} a_j^{\tensor 2}\iprod{a_j,v}^2\mper
	\end{align*}
	Therefore, defining $M$ to be the $n \times n$ reshaping of $T(v\tensor v)$ and defining $N$ to be the $n \times n$ reshaping of $E(v\tensor v)$,
	\[
	    M = (1-\alpha)a_1a_1^\top + \sum_{j > 1} \iprod{a_j,v}^2 a_j a_j^{\top} + N,
	\]
	where $\|N\|_F \le \epsilon$, since $\|E\|\le \epsilon$.

	Now, we have that
	\[
	    \|M\| \ge  1-\alpha - \epsilon,
	\]
	and that if we choose $\eta$ so that $1-\eta \ge 2\alpha \ge 1/50$,
	\[
	    \|M - \eta \cdot a_1a_1^\top\|
	    \le 1-\alpha- \eta + \epsilon.
	\]
	Thus, $\|M - \eta a_1 a_1^\top\| \le \|M\| - \eta + 2\epsilon$, and we have by \pref{fact:topeig} (see \pref{app:tools} for a proof) that the top unit eigenvector $u$ of $M$ is such that
	\[
	    \iprod{u,a_1}^2 \ge \frac{\eta - 2\epsilon}{\eta} \ge 1 - \frac{2\epsilon}{\eta}.
	    \]
	    Choosing $\eta = 49/50$, the result follows.
    \end{proof}

\subsection{Near-orthonormal components}\label{sec:northo}

We'll now dispense with the discrepancy between the orthonormal and near-orthonormal cases.
\begin{fact}\label{fact:orthog}
    If $S = \sum_i a_i a_i^{\top} = \Id + E$ for $\|E\| \le \epsilon$, then
    $\tilde{a}_1 =S^{-1/2}a_1,\ldots,\tilde{a}_n = S^{-1/2}a_{n}$ are orthonormal, $\iprod{\tilde{a_i},a_i}^2 \ge (1-\epsilon)\|a_i\|^2$, and
    \[
	\left\|\sum_i \tilde{a}_i^{\tensor 4} - a_i^{\tensor 4}\right\| \le 4\sqrt{\epsilon}.
    \]
\end{fact}
\begin{proof}
    The fact that the $\tilde{a}_i$ are orthonormal follows because they are independent and have Gram matrix $\Id$.
    Using the fact that the eigenvalues of $S$ are between $(1-\epsilon)^{-1/2}$ and $(1+\epsilon)^{-1/2}$, quantity $\iprod{a_i,\tilde{a}_i}^2 = \left(a_i^\top S^{-1/2}a_i\right)^2 \ge \frac{\|a_i\|^2}{1+\epsilon} \ge \|a_i\|^2(1-\epsilon)$.
    Finally,
    \begin{align*}
	\sum_i \tilde{a}_i^{\tensor 4} - a_i^{\tensor 4}
	&= (S^{-1/2})^{\tensor 2}\left(\sum_i a_i^{\tensor 4}\right)(S^{-1/2})^{\tensor 2}-\sum_i a_i^{\tensor 4}\\
    \end{align*}
    and because $\|(S^{-1/2})^{\tensor 2} - \Id\| \le \sqrt{\epsilon}$, and $\sum_i a_i^{\tensor 4} \preceq \sum_{ij} a_ia_i^{\top}\tensor a_ja_j^{\top}$, this difference has spectral norm at most $3\epsilon\|\sum_i a_i^{\tensor 4}\| \le 3\sqrt{\epsilon} (1+\epsilon)^2 \le 4\sqrt{\epsilon}$.
\end{proof}

\subsection{Full Recovery}\label{sec:fullrecovery}
    Now we give the full algorithm, which will remove the components we find in each step from the tensor without amplifying the spectral norm of the error too much.
\begin{algorithm}[Full tensor decomposition]\label{alg:overall}
    {\color{white}.}\\
    {\em Input:} A tensor $\bT \in (\R^d)^{\tensor 4}$, and the error parameter $\epsilon \ge \|E_{\{1,2\}\{3,4\}}\|$.
    \begin{compactenum}
    \item Initialize the set of known components $K = \emptyset$, and the set of components under inspection $B=\emptyset$.
    \item Initialize a working copy of $\bT$, $\bT^{(0)}_{work}$, and keep a clean copy of $\bT$ called $\bT_{clean}$.
    \item For $t = 0,\ldots, 100\log n$,
	\begin{compactenum}
	\item  Preprocess $\bT^{(t)}_{work}$ with \pref{alg:preproc}, then run \pref{alg:orthog} with $\bT_{work}^{(t)}$ $\tO(n)$ times, then postprocess with \pref{alg:post} using $\bT_{clean}$ and error parameter $\epsilon$.
	    If this produces an output vector $v$, add $v$ to $B$ (unless $K$ already contains a vector that is $1-\eps$ correlated with $v$).
	\item Let $B = b_1,\ldots,b_m$, and abuse notation by letting $B$ also be the matrix whose $i$th row is $b_i$.
	    Compute the singular value decomposition $B = U \Sigma V^\top$, and compute the orthonormalized set $\tilde{B} = \{\tilde{b}_i\} = \{U\Sigma^{-1}U^\top b_i\}_{i=1}^m$.
	\item Remove from $\tilde{B}$ any $\tilde{b_i}$ for which $\iprod{\bT_{clean},\tilde{b_i}^{\tensor 4}} < (1-6\epsilon)^2 - \epsilon$.
	\item Update the known components: set $K := K\cup \tilde{B}$, and set $B,\tilde{B} := \emptyset$.
	\item Update the working tensor by removing known components: set $\bT^{(t+1)}_{work} := \bT^{(t)}_{work} - \sum_{\tilde{b} \in \tilde{B}} \tilde{b}^{\otimes 4}$.
	\end{compactenum}
    \end{compactenum}
    {\em Output:} The set of known components $K$.
\end{algorithm}

\begin{theorem}\label{thm:overall}
    Given $\bT = \sum_{i=1}^n a_i^{\tensor 4} + E$ where the $a_i$ are orthonormal and $\|E_{\{1,2\}\{3,4\}}\| \le \epsilon$, then if $\epsilon < O(\eta^2/\log^2 n)$,
    with probability $1-o(1)$, \pref{alg:overall} recovers orthonormal vectors $b_1,\ldots,b_n$ so that there exists a permutation $\pi:[n]\to [n]$ such that for each $i \in [n]$,
    \[
	\iprod{a_i,b_{\pi(i)}}^2 \ge 1-3\epsilon.
    \]
Furthermore, this requires runtime $\tO(n^{1+O(\eta)}d^{2+\omega})$.
\end{theorem}

First, we prove that if we have an \emph{orthonormal} basis that approximates $a_1,\ldots,a_k$, we can subtract it without introducing a large spectral norm error---this motivates and justifies steps 3(b)--3(e).

\begin{lemma}\label{lem:subtraction}
    Let $a_1,\ldots,a_k\in \R^d$ and $b_1,\ldots,b_k\in \R^d$ be two sets of orthonormal vectors, such that $\iprod{a_i,b_i}^2 \ge 1-\epsilon$.
    Then
    \[
	\left\|\sum_i a_i^{\tensor 4} - b_i^{\tensor 4} \right\|_2 \le 4\sqrt{\epsilon}
    \]
\end{lemma}
\begin{proof}
    Define the matrices $U,V \in \R^{d^2 \times k}$ so that the $i$th column of $U$ (or $V$) is equal to $a_i^{\tensor 2}$ ($b_i^{\tensor 2}$ respectively).
    We have that
    \[
	\sum_i a_i^{\tensor 4} - b_i^{\tensor 4}
	= UU^\top - VV^\top = (U-V)(U+V)^\top.
    \]
    So it suffices for us to bound $\|U-V\|\cdot \|U+V\|$.

    By the subadditivity of the norm, $\|U+V\| \le \|U\| + \|V\| = 2$.
    Meanwhile, the singular values of $U-V$ are the square roots of the eigenvalues of $\|(U-V)^\top(U-V)\|$, and so we bound
    \begin{align}
	(U-V)^\top (U-V)
	&= UU^\top + VV^\top - U^\top V - V^\top U\nonumber\\
	&= 2\Id_k - U^\top V - V^\top U\mcom\label{eq:sqbd}
    \end{align}
    where the second line follows because $U$ and $V$ have orthonormal columns.
    Now, by assumption we know that
    \[
	U^\top V = (1-\epsilon)\cdot \Id_k + E,
    \]
where for $i\neq j$, $E_{ij} = \iprod{b_i,a_j}^2$ and $E_{ii} = \iprod{b_i,a_i}^2 - (1-\epsilon)$, and by the orthonormality of the $a_j$,
    \[
	\sum_j |E_{ij}| = \sum_j \iprod{b_i,a_j}^2 = \epsilon.
    \]
So the $1$-norm of the rows of $E$ is at most $\epsilon$.
By the orthonormality of the $b_j$, the same holds for the $1$-norm of the columns, $\sum_i |E_{ij}|$.
It follows that $\|E\| \le \epsilon$.
Therefore,
    \[
	U^\top V + V^\top U = 2(1-\eps)\Id_k + \hat {E},
    \]
where $\|\hat{E}\|\le 2\epsilon$.
    Returning to \pref{eq:sqbd}, we can conclude that $\|(U-V)\| \le 2\sqrt{\epsilon}$, and we have our result.
\end{proof}

Now, we will prove that by orthogonalizing, we do not harm too many components $\tilde{b}_i$.
\begin{lemma}\label{lem:orthonormalize}
    Suppose $a_1,\ldots,a_k \in \R^d$ are orthonormal vectors, and $u_1,\ldots,u_k \in \R^d$ are unit vectors such that $\|u_i - a_i\|_2^2 \le \epsilon$.
    Let $U$ be the $d\times k$ matrix whose $i$th column is $u_i$, $U = X\Sigma Y$ be the singular value decomposition of $U$, and let $\tilde u_i = X\Sigma^{-1}X^\top u_i$.
    Then for a $1-\delta$ fraction of $i\in[k]$,
    \[
	\iprod {u_i,\tilde{u}_i} \ge 1-\epsilon/2\delta.
    \]
\end{lemma}
\begin{proof}
    For convenience, let $A$ be the $d \times k$ matrix whose $i$th column is $a_i$, and let $\tilde{U} = X\Sigma^{-1}X^\top U$, and $\tilde u_i = X\Sigma^{-1}X^\top u_i$.
    Let $\mathbb{X}$ be the space of all real $d\times k$ real matrices with orthonormal columns, and notice that $\tilde{U}, A \in \mathbb{X}$ and that $\tilde{U}$ is closer to $U$ than $A$.
    Indeed, for any matrix $X$ with orthonormal columns,
    \[
	\|X - U\|_F^2
	~=~ k + \|U\|_F^2 - 2 \iprod{U,X}
	~\ge~ k + \|U\|_F^2 - 2\|X\|\|U\|_*
	~=~ k + \|U\|_F^2 - 2\|U\|_*
	~=~ \|\tilde{U} - U\|_F^2,
    \]
    Where we have used that the spectral norm and nuclear norm are dual.
    Therefore,
    \begin{align*}
	\|\tilde{U}-U\|_F^2
	~\le~ \|A - U\|_F^2
	~=~ \sum_i \|a_i - u_i\|_2^2
	~=~ \epsilon \cdot k
    \end{align*}
    And on average, $\epsilon \ge \|\tilde{u}_i - u_i\|_2^2$, so by Markov's inequality, for at least $(1-\delta)k$ of the $u_i$, $\iprod{u_i,\tilde u_i} \ge 1 - \epsilon/2\delta$.
\end{proof}

Finally, we are ready to prove that \pref{alg:overall} works.

\begin{proof}[Proof of \pref{thm:overall}]
    We claim that in the $t$th iteration of step $3$, with high probability we have at most $0.45^{t}n$ components remaining to be found, and that $\bT_{work}^{(t)} = \bT + F^{(t)}$ where $\|F^{(t)}\| \le 8t \sqrt{\epsilon}$.
    For $t = 0$, this is easily true.

    Now assume this holds for $t$, and we will prove it for $t+1$.
    Since by assumption $t\sqrt{\epsilon}\log n \le 100\sqrt{\epsilon}\log n \le O(\eta)$, applying \pref{lem:step1} and \pref{thm:orthog} to the running of preprocessing \pref{alg:preproc} and the main step \pref{alg:orthog} with $\bT_{work}^{(t)}$ and \pref{lem:postproc} to the running of the postprocessing \pref{alg:post} with $\bT_{clean}$, in step $3(a)$ with high probability we will find $m \ge 0.9n_t$ vectors $b_1,\ldots,b_m$ so that $\iprod{b_i,a_i}^2 \ge 1-3\epsilon$.
    Furthermore, this takes a total of $\tO(m n^{O(\eta)} d^{2+\omega})$ time.

    By \pref{lem:orthonormalize}, in step $3(c)$ we will remove no more than a half of the $\tilde{b}_i$, while maintaining $\iprod{\tilde{b}_i,a_i}^2 \ge 1-3\epsilon$ (where we are abusing notation by re-indexing conveniently), so that $n_{t+1} \ge 0.45 n_t$.
    Finally, by \pref{lem:subtraction}, we have that
    \[
	\left\|\sum_{i=1}^{|\tilde B|} a_i^{\tensor 4} -\tilde{b_i}^{\tensor 4}\right\|_2 \le 4\sqrt{3\epsilon},
    \]
So that in step 3(e),
    \begin{align*}
	\bT_{work}^{(t+1)}
	&= \left(\bT_{work}^{(t)} - \sum_{i=1}^{|\tilde B|} a_i^{\otimes 4}\right) + \left(\sum_{i=1}^{|\tilde B|} a_i^{\otimes 4} - \sum_{\tilde{b}_i \in \tilde{B}}\tilde{b}_i^{\tensor 4} \right)\\
	&= \bT_{work}^{(t-1)} + F,
    \end{align*}
    for a matrix $F$ with $\|F\| \le 8\sqrt{\epsilon}$.
    By induction, this implies that $\bT_{work}^{(t+1)} = \bT + F^{(t+1)}$ where $\|F^{(t+1)}\| \le \|F\| +\|F^{(t)}\| \le (t+1)8\sqrt{\epsilon}$.

    Taking a union bound over the high-probability success of \pref{alg:orthog}, we have that after $t = O(\log n)$ steps we have found all of the components.
    We have spent a total of $\tO(n^{1+O(\eta)}d^{2+\omega})$ time in step 3(a).
    Finally, steps 3(b)-3(e) of \pref{alg:overall} require no more than $\tO(d^3)$ time, and since the entire loop runs $\tO(1)$ times, we have our result.
\end{proof}

\subsection{Supporting Lemmas}\label{sec:supporting}

Now we circle back and prove the omitted supporting lemmas.

\begin{proof}[Proof of \pref{lem:algsuccess}]
    For convenience, fix $j := 1$.
    Let $g^{(1)}$ be the component of $g$ in the direction $a_1^{\otimes 2}$, and let $g^{(>1)}$ be the component of $g$ orthogonal to $a_1^{\otimes 2}$.
    Notice that $g^{(1)},g^{(>1)}$ are independent.

    By the orthogonality of the $a_i$, our matrix $M_g$ can be written as
    \begin{align*}
	M_g
	&= \iprod{g^{(1)},a_1^{\otimes 2}}\cdot a_1 a_1^{\top} + \sum_{j=1}^{d^2} (g^{(1)}_j+g^{(>1)}_j) \cdot (S - a_i^{\otimes 4} + E)_j\\
	&= \iprod{g^{(1)},a_1^{\otimes 2}}\cdot a_1 a_1^{\top} + \left(\sum_{j=1}^{d^2} g^{(1)}_j \cdot E_j\right) + \left(\sum_{j=1}^{d^2} g^{(>1)}_j \cdot T_j\right)\mcom
    \end{align*}
where $T_j$ is the $j$th matrix slice of $T$.
    For convenience, we can refer to the two sums on the right as
    \[
	N =  \left(\sum_{j=1}^{d^2} g^{(1)}_j \cdot E_j\right) + \left(\sum_{j=1}^{d^2} g^{(>1)}_j \cdot T_j\right).
    \]
\medskip

    First, we get a lower bound on the probability that the coefficient of $a_1a_1^{\top}$ is large.
Let $\cG_1(\alpha)$ be the event that $|\iprod{g^{(1)},a_1^{\otimes 2}}| = \|g^{(1)}\|\ge \sqrt{2\alpha\log n}$.
    By standard tail estimates on univariate Gaussians, we have that
    \[
	\Pr[\cG_1(\alpha)] \ge \tilde O(n^{-\alpha}).
    \]

Now, we bound $\|N\|$.
Define the event $\cE_{>1}(\rho)$ to be the even that
    \[
	\cE_{>1}(\rho) \defeq \left\{\left\|\sum_{j=1}^{d^2} g^{(>1)}_j \cdot T_j\right\|
	\le \sqrt{2(1+\rho)\log d}\right\}
    \]
    By \pref{lem:gflat}, we can conclude that
    \[
	\Pr\left[\cE_{>1}(\rho)\right] \ge 1 - d^{-\rho}\mper
    \]

    To bound $\|N\|$, it thus remains to understand the term
    \begin{align}
	\sum_j g_j^{(1)}E_j
	= \iprod{g,a_1^{\otimes 2}}\cdot \sum_j a_i^{\otimes 2}(j) \cdot E_j
	= \iprod{g,a_1^{\otimes 2}} \cdot (a_ia_i^\top \otimes \Id_{d^2})E,
	\label{eq:equality}
    \end{align}
where the quantity $(a_ia_i^\top \otimes \Id_{d^2})E$ corresponds to the contraction of $E$ along two modes by the vector $a_i^{\otimes 2}$.
We make the following observation:
    \begin{observation}\label{obs:frobs}
	If $P_1,\ldots,P_n$ are orthogonal projections from $\R^{n^4} \to K$ for some convex set $K$, then for a $1-\delta$ fraction of $i \in [n]$,
	\[
	    \|P_iE\|_F \le \epsilon/\delta.
	\]
    \end{observation}
    \begin{proof}
	This follows from the fact that
	\begin{align*}
	    \epsilon^2 n
	    &\ge \|E\|_F^2
	    ~\ge \sum_i \|P_i E\|_F^2,
	\end{align*}
	and then by an application of Markov's inequality.
    \end{proof}
    Now, note that $\sum_j a_ia_i^\top \tensor \Id_{d^2}$ for $i \in [n]$ are orthogonal projectors from $\R^{n^4}$ to $\R^{n^2}$.
    Thus it follows that for a $1-\delta$ fraction of $i \in [n]$, and without loss of generality assuming that $i=1$ is among them, $\|(a_1a_1^{\top}\otimes \Id_{d^2})E\|_F \le \epsilon/\delta$.
    Therefore for any unit vectors $u,v \in \R^d$, returning to \pref{eq:equality},
    \begin{align*}
	\left|u^{\top}\left(\sum_{j=1}^{d^2} g_{j}^{(1)} \cdot E_j\right)v\right|
	&= \left|\iprod{g,a_1^{\otimes 2}} \cdot \iprod{uv^\top, (a_1a_1^\top \tensor \Id_{d^2})E}\right|\\
	&\le \|g^{(1)}\|\cdot \left\|(a_1a_1^{\top}\otimes \Id_{d^2})E\right\|_F \cdot \|uv^\top\|_F
	~\le~ \frac{\epsilon}{\delta}\|g^{(1)}\|\mper
    \end{align*}
    Thus, combining the above we have a two-part upper bound on $\|N\|$.

Finally, define the event $\cE_{a_1,E}(\theta)$ to be the event that
    \[
	\cE_{a_1,E}(\theta) \defeq \left\{\left\|\left(\sum_{j=1}^{d^2} g^{(>1)}_j \cdot T_j\right)a_1\right\|_2, \left\|\left(\sum_{j=1}^{d^2} g^{(>1)}_j \cdot T_j\right)^{\top} a_1\right\|_2
	\le \frac{\epsilon}{\delta}\cdot \sqrt{2(1+\theta)}\right\}
    \]
    Examining this form, we can split
    \begin{align*}
	\left(\sum_{j=1}^{d^2} g^{(>1)}_j \cdot T_j\right)a_1
	&=\sum_{j=1}^{d^2} g^{(>1)}_j \cdot (S-a_1^{\otimes 4})_ja_1 +\sum_{j=1}^{d^2} g^{(>1)}_j \cdot E_ja_1\\
	&=\sum_{j=1}^{d^2} g^{(>1)}_j \cdot E_ja_1
    \end{align*}
    where the last line follows because the $a_i$ are orthogonal.
    We note that $\sum_j g^{(>1)}_j E_j a_1$ is a Gaussian contraction of the form $(a_1 \tensor \Id_{d^3})E$.
    Again appealing to \pref{obs:frobs} and to the fact that the $a_i \tensor \Id_{d^3}$ are orthogonal projections, we conclude that for a $1-\delta$ fraction of $i \in [n]$, $\|(a_1 \tensor \Id_{d^3})E\|_F \le \epsilon/\delta$.
    Without loss of generality we assume that this is true for $i=1$, from which it follows by \pref{lem:gflat} that
    \[
	\Pr\left[\left\|\sum_{j=1}^{d^2} g^{(>1)}_j \cdot E_ja_1\right\|_2 \le \frac{\epsilon}{\delta}\sqrt{2(1+\theta)}\right] \ge 1 - d^{-\theta}.
    \]
    We can apply the same arguments to $\left(\sum_j g_j^{(>1)}E_j\right)^\top a_1$, and we conclude that
    \[
	\Pr\left[\cE_{a_1,E}(\theta)\right] \ge 1-2d^{-\theta}.
    \]

Now, by the union bound $\cE_{>1}(\rho)$ and $\cE_{a_1,E}$ both occur with probability at least $1-d^{-\rho}-2d^{-\theta}$.
    Also, we notice that $\cE_{>1}\cup \cE_{a_1,E}$ and $\cG_{1}$ are independent.
    Therefore, for $\rho,\theta \ge \log \log n/\log n$,
    \[
	\Pr[\cG_{1}(\alpha),\cE_{>1}(\rho),\cE_{a_1,E}(\theta)] \ge \tilde O(n^{-\alpha})\mper
    \]
Conditioning on $\cE_{>1}$ and $\cG_{1}$,
    \[
	M_g = c \cdot a_1 a_1^{\top} + N,
    \]
where $|c| \ge \sqrt{2\alpha\log n}$, and $N$ is a matrix of norm at most $(\epsilon/\delta) c + \sqrt{2(1+\rho)\log d}$, such that $\|N a_1\|, \|N^\top a_1\| \le (\epsilon/\delta) (c+\sqrt{2(1+\theta)})$.

    We now set $\alpha$ so that $|c| \ge \beta\|N\|$.
    This occurs when
    \begin{align*}
	\alpha
	&\ge \left(\beta\frac{1}{1-\beta(\epsilon/\delta)}\right)^2(1+\rho).
    \end{align*}
    Choosing $\beta = 1 + \beta'$, $\rho,\theta = \log\log n/\log n$, we have our conclusion for $a_1$, and by symmetry for all other $a_i$ in the $1-3\delta$ fraction of $i \in [n]$ for which the Frobenius norms of the contractions are small.
\end{proof}

\begin{proof}[Proof of \pref{lem:topeig}]
    Assume without loss of generality that $c \ge 0$.
    Choose $\kappa = \frac{2\epsilon|c|}{\delta} \le |c|\left(1-\frac{1}{1+\beta}\right)$.
    When $\kappa\cdot a_1a_1^{\top}$ is subtracted from $M_g$,
    then given any unit vector $v \in \R^d$ with $|\iprod{v,a_1}| = \alpha$, we can write $v = \alpha a_1 + w$ where $\iprod{w,a_1} = 0$ and $\|w\| = \sqrt{1-\alpha^2}$.
    Examining the action of $M_g - \kappa a_1a_1^\top$ on $v$,
    \begin{align*}
	\|(M_g - \kappa\cdot  a_1 a_1^\top)v\|_2^2
	&= \| (c - \kappa)\alpha a_1 + Nv\|_2^2 \\
	&\le (c-\kappa)^2\alpha^2 + (c-\kappa)\alpha a_1^\top N v + (c-\kappa)\alpha v^\top N a_1 + v^\top N^\top N v \\
	\intertext{applying the Cauchy-Schwarz inequality and our bounds on $\|N^\top a_1\|,\|Na_1\|$,}
	&\le (c-\kappa)^2\alpha^2 + 2(c-\kappa)\alpha\epsilon c + v^\top N^\top N v \mper
    \end{align*}
	Now expanding the $v^\top N^\top N v$ term along the components of $v$,
    \begin{align*}
	v^\top N^\top N v
	&= (\alpha a_1 + w)^\top N^\top N (\alpha a_1 + w)\\
	&\le (\alpha\|Na_1\| + \|w\|\|N\|)^2
	\intertext{and since $\|w\| = \sqrt{1-\alpha^2}$, $\|Na_1\| \le \epsilon c$, and $\|N\| \le c/(1+\beta)\le c(1-\beta+2\beta^2)$,}
	&\le \left(\alpha\epsilon c + \sqrt{1-\alpha^2}(1-\beta+2\beta^2)c\right)^2\mcom
    \end{align*}
    and putting these together,
    \[
\|(M_g - \kappa\cdot  a_1 a_1^\top)v\|_2^2
    \le (c-\kappa)^2\alpha^2 + 2(c-\kappa)\alpha\epsilon c + \left(\alpha\epsilon c + \frac{\sqrt{1-\alpha^2}}{1+\beta}c\right)^2
    \]
It is easy to see that when $c/(1+\beta) < c-\kappa$, this quantity is maximized at $\alpha = 1$, and so by our choice of $\kappa$ we have that
    \[
\|(M_g - \kappa\cdot  a_1 a_1^\top)v\|_2^2
    \le (c-\kappa)^2 +2\epsilon c(c-\kappa) +\epsilon^2 c^2 = (c(1+\epsilon) -\kappa)^2
    \]
and thus $\|M_g -\kappa a_1a_1^\top\| \le (1+\epsilon)c - \kappa$.

    Now we will lower bound $\|M_g\|$.
    \begin{align*}
	\|M_g\|
	~\ge~ a_1^{\top}M_ga_1
	&= c + a_1^{\top}Na_1\\
	&\ge c -  \|a_1\|\|Na_1\|\\
	&\ge c(1-\epsilon)\mper
    \end{align*}
	Where we have applied the Cauchy-Schwarz inequality, and the assumption that $\|Na_1\|\le \epsilon c$.
	It follows that
	\[
	    \|M_g - \kappa a_1 a_1^\top\| \le \|M_g\| + 2\epsilon c - \kappa.
	\]
Finally applying \pref{fact:topeig}, we can conclude that for either the left- or right-singular unit vector $u$ of $M_g$,
    \begin{align*}
	\iprod{a_i,u}^2
	&\ge \frac{\kappa - 2\epsilon c}{\kappa}
	~\ge  1 - \delta\mper
    \end{align*}
    Choosing $\delta = \frac{2\epsilon(1+\beta)}{\beta}$ as small as possible, we have our result.
\end{proof}
\section{Learning Orthonormal Dictionaries}
Here, we show how to use our tensor decomposition algorithm to learn dictionaries with orthonormal basis vectors.

\begin{problem}
Given access to a dictionary $A \in \R^{d\times d}$ with independent columns $a_1,\ldots,a_d$, in the form of samples $y^{(1)} = Ax^{(1)},\ldots,y^{(m)} = Ax^{(m)}$ for independent $x^{(i)}$, recover $A$.
\end{problem}

Below, in \pref{sec:samplecomplex}, we will prove that $\tO(n^3)$ samples suffice to estimate the $4$th moment tensor within $o(1)$ spectral norm error.
Computing this matrix from $\tO(n^3)$ samples takes $\tO(n^3d^4)$ time.
Thus we can equivalently formulate the problem as follows:

\begin{problem}
    Given access to $A \in \R^{d\times n}$ with independent columns $a_1,\ldots,a_n$, via a noisy copy of the $4$th moment tensor $\bT = \E[(Ax)^{\tensor 4}] + E$, recover the columns $a_1,\ldots, a_n$.
\end{problem}

Finally, given access to a sufficiently large number of samples, we can reduce to the case where the columns of $A$ are orthogonal:
\begin{lemma}\label{lem:indtoorth}
    Suppose that the samples $y = Ax$ are generated from a distribution over $x$ for which $\E[x_i^2] = \E[x_j^2]$ for all $i,j\in[n]$, and that the columns of $A$ are independent.
    Then there exists an efficient reduction from the case when $A$ has independent columns to the case when $A$ has orthogonal columns, with sample complexity growing polynomially with the condition number of $\Sigma = AA^\top$.
\end{lemma}
The proof is straightforward, involving a transformation by the empirical covariance matrix, and we give it below in \pref{sec:indtoorth}.

\medskip

Now, we reduce the dictionary learning problem to tensor decomposition.
In \pref{sec:techniques}, we explained that the  $4$th moment tensor itself may be far from our target tensor $\sum_i a_i^{\tensor 4}$, due to the presence of a low-rank, high-Frobenius norm component.
The sum-of-squares algorithm for this problem can overcome this difficulty by exploiting the SDP's symmetry constraints.
In our algorithm, we will exploit this symmetry manually to go from $\E[(Ax)^{\tensor 4}]$ to a tensor that approximates $\sum_i a_i^{\tensor 4}$ well in Frobenius norm.

\begin{algorithm}\label{alg:cleanmoment}
    {\bf Input:} A noisy copy of the fourth moment tensor $\bT = \E[(Ax)^{\tensor 4}] + E$, truncation parameter $\eps$.
    \begin{compactenum}
    \item Reshape $\bT$ to $T_{\{1,2\}\{3,4\}}$, and perform the eigenvalue truncation
	    \[
		T^{>\eps} \defeq (T-\eps\Id)_+,
	    \]
	    where $+$ denotes projection to the PSD cone.
	\item Compute the truncation of a different reshaping of $T^{>\eps}$,
	    \[
		\tilde{T} = (\sigma(T^{>\eps}_{\{1,3\}\{2,4\}}) - \eps\Id)_+
	    \]
    \end{compactenum}
    {\bf Output:} The tensor $\tilde{T}$ as an approximation of $\sum_i \E[x_i^4]\cdot a_i^{\tensor 4}$.
\end{algorithm}

Our claim is that this produces a tensor that is close to $\sum_{i} \E[x_i^4]\cdot a_i^{\tensor 4}$ in Frobenius norm.

\begin{theorem}
    If the $x$ are independent and distributed so that $\E[x_ix_jx_kx_\ell] = 0$ unless $x_ix_jx_kx_\ell$ is a square, and so that for all $i,j\in[n]$, $\E[x_i^2x_j^2]\le \alpha \E[x_1^4]$ for $\alpha < 1$, then given access to $\bT = \E[(Ax)^{\tensor 4}] + E$ where $\|E\| \le \alpha$,
    \pref{alg:cleanmoment} with $\epsilon = 3\alpha$ returns a tensor $\tilde{T}$ such that
    \[
	\left\|~\tilde{T} - \sum_i \E[x_i^4] \cdot a_i^{\tensor 4}\right\|_F \le 9\alpha \sqrt{n}.
	\]
\end{theorem}

\begin{proof}
    For convenience denote by $T \defeq (\E[(Ax)^\tensor 4] + E)_{\{1,2\}\{3,4\}}$, and define $S = \sum_{i} \E[x_i^4] \cdot a_i^{\tensor 4}$.
    We have that
    \begin{align*}
	T - E ~=~ \E[(Ax)^{\tensor 4}]
	&= \sum_{i,j,k,\ell=1}^d \E[x_ix_jx_kx_\ell]\cdot (a_i\tensor a_j)(a_k\tensor a_\ell)^\top\\
	&= \sum_{i}\E[x_i^4] \cdot a_i^{\tensor 4} + \sum_{i\neq j} \E[x_i^2x_j^2] \cdot \left((a_i^{\tensor 2})(a_j^{\tensor 2})^\top + a_ia_i^\top \tensor a_ja_j^\top +  a_ja_i^\top \tensor a_ia_j^\top \right)
    \end{align*}
    The latter term can be split into three distinct matrices---the first is a potentially low-rank matrix, and may have large eigenvectors.\footnote{For instance, in the case when $\E[x_i^2x_j^2] = p^2$, this term is rank-$1$ and has spectral norm $pn$.}
    The latter two terms have small spectral norm.

    \begin{claim}\label{claim:spectrals}
	\[
	    \left\|\sum_{i\neq j} \E[x_i^2x_j^2] \cdot a_ia_i^\top \tensor a_ja_j^\top\right\|\le \alpha\E[x_1^4], \quad \text{and}\quad \left\|\sum_{i\neq j} \E[x_i^2x_j^2]\cdot a_ja_i^\top \tensor a_ia_j^\top\right\| \le \alpha\E[x_1^4]\mper
	    \]
    \end{claim}
    We'll prove this claim below.
Now, again for convenience define the matrix $N$ to be the remaining term, $N = \sum_{i\neq j} \E[x_i^2x_j^2] (a_i^{\tensor 2})(a_j^{\tensor 2})^\top$.
    From \pref{claim:spectrals} and by our assumption on $\|E\|$,
    \[
	T = S + N + \hat{E},
    \]
for $\|\hat{E}\| \le 3\alpha$.
    On the other hand, if we let $B$ be the $d \times n$ matrix whose $i$th column is $a_i^{\tensor 2}$, and we let $X$ be the $n \times n$ matrix whose $i,j$th entry is $\E[x_i^2x_j^2]$, then
    $S + N = BXB^\top$, and so $\rank(S+N) \le n$.

    It follows that when we perform the eigenvalue truncation in step 1 of \pref{alg:cleanmoment},
    \[
	T^{<\eps} = (T - 3\alpha \cdot \Id)_+,
    \]
then we have that $\rank(T^{<\eps}) \le n$ as well.
Also by definition of truncation, $T = T^{<\eps} + \tilde{E}$, and because to begin with we had $T\succeq 0$, $\|\tilde{E}\| \le 3\alpha$.
Putting the above together, it follows that
    \[
	\|T^{<\eps} - (S+N)\|_F = \|\tilde{E} - \hat{E}\|_F \le 6\alpha \sqrt{2n},
    \]
where we have used that $\rank(T^{<\eps} - (S+N)) \le 2n$ and $\|\tilde{E} - E\| \le 6\alpha$.
    Now, we recall the reshaping operation on tensors from step 2 of \pref{alg:cleanmoment}---in going from the reshaping $\{1,2\}\{3,4\}$ to $\{1,3\}\{2,4\}$, the rank-1 tensor $(a \tensor b)(c \tensor d)^\top)$ is reshaped to $(a\tensor c)(b\tensor d)^\top$.
    Let $\sigma(\cdot)$ denote this reshaping operation.
    Reshaping does not change the Frobenius norm.
    So by linearity, and since $\sigma$ fixes $S$,
    \[
	\|\sigma(T^{<\eps}) - S - \sigma(N)\|_F = \|T^{<\eps} - (S+N)\|_F \le 6\alpha\sqrt{2n}.
    \]
We now remark that $\sigma$ maps $N$ to one of the bounded-norm matrices from \pref{claim:spectrals},
    \[
	\sigma(N) =  \sum_{i\neq j} \E[x_i^2x_j^2]\cdot a_i a_i^\top \tensor a_j a_j^\top \preceq \alpha \cdot \Id.
	\]

	Furthermore, because the positive semidefinite cone is a closed convex set, and because projection to closed convex sets can only decrease distances (see \pref{lem:proxproj}),
	\begin{align*}
	\left\|\sigma(T^{<\eps}) - S - \sigma(N) \right\|_F
	    &= \left\|\sigma(T^{<\eps}) - S - \alpha\cdot \Id - (\sigma(N) - \alpha\cdot \Id)\right\|_F\\
	    &\ge \left\|\left(\sigma(T^{<\eps}) - S - \alpha\cdot \Id\right)_+ - (\sigma(N) - \alpha\cdot \Id)_+ \right\|_F \\
	    &\ge \left\|\left(\sigma(T^{<\eps}) - S - \alpha\cdot \Id\right)_+\right\|_F\\
	    &\ge \left\|\left(\sigma(T^{<\eps}) - \alpha\cdot \Id\right)_+- S \right\|_F\mcom
	\end{align*}
	where to obtain the last inequality we used that $S$ is positive semidefinite.
	Therefore, step $3$ of the algorithm ensures that $\tilde T$ is close to $S$ in Frobenius norm, as desired.
\end{proof}

Now we prove that the spectral norms of the symmetrizations of the tensor have small spectral norm.
\begin{proof}[Proof of \pref{claim:spectrals}]
The first matrix that we are interested in is PSD, and can dominated by a tensor power of the identity:
    \begin{align*}
	0\preceq
	\sum_{i\neq j} \frac{\E[x_i^2x_j^2]}{\E[x_1^4]}\cdot a_i a_i^\top \tensor a_j a_j^\top
	\preceq \alpha \sum_{i,j} a_i a_i^\top \tensor a_j a_j^\top
	\preceq \alpha \cdot \left(\sum_i a_i a_i^\top\right) \tensor \left(\sum_j a_j a_j^\top \right)
	\preceq \alpha \cdot \Id\mper
    \end{align*}
    For the second matrix, if we let $A$ be the $d^2 \times n^2$ matrix whose $i,j$th column is $a_i \tensor a_j$, and let $M$ be the $n^2 \times n^2$ matrix whose $i,j$th diagonal entry is $\E[x_i^2 x_j^2]$, then
    \begin{align*}
	\sum_{i\neq j} \E[x_i^2x_j^2]\cdot a_j a_i^\top \tensor a_i a_j^\top
	= AM\Pi A^\top,
    \end{align*}
    where $\Pi$ is the permutation matrix that takes the $i,j$th row to the $j,i$th row.
    By assumption, $\|M\| \le \max_{i\neq j} \E[x_i^2x_j^2] \le \alpha \E[x_1^2]$,
    and the columns of $A$ are orthonormal, so $\|A\| = 1$.
    It follows by the submultiplicativity of the spectral norm that
    \[
	\left\|\sum_{i\neq j} \E[x_i^2x_j^2]\cdot a_j a_i^\top \tensor a_i a_j^\top\right\| \le \alpha\E[x_1^4]\mper
    \]
This gives us the claim.
\end{proof}

When $\E[x_i^4] = p$ for all $i \in [n]$, applying \pref{alg:orthog} with $\tilde T$ a total of $\tO(n)$ times will allow us to recover $m \ge n/2$ vectors $b_1,\ldots,b_{m}$ with $\iprod{a_i,b_i}^2 \ge 0.99$.
The following subsections contain the details regarding the refinement of the approximation, and the sample complexity bounds for estimating the 4th moment tensor.

\subsection{Postprocessing to refine approximation}

We now analyze the postprocessing algorithm \pref{alg:post} for the context of dictionary learning, in which our tensor has the form $\bT = \E[(Ax)^{\tensor 4}]$.
We claim that, despite not having bounded spectral norm error away from $\sum_{i} \E[x_i^4] \cdot a_i^{\tensor 4}$, the postprocessing algorithm still succeeds.

\begin{lemma}
    Suppose that we are given $\bT = \E[(Ax)^{\tensor 4}]$, where $x$ is distributed so that $\E[x_ix_jx_kx_\ell] = 0$ unless $x_ix_jx_kx_\ell$ is a square, $\E[x_i^4] = p$ for all $i\in[n]$, and $\E[x_i^2x_j^2]\le \alpha p$.
    Suppose furthermore that we have a unit vector $u$ such that $\iprod{b,a_i}^2 \ge 0.99$ for some $i\in[n]$.
    Then applying \pref{alg:post} to $u$ and $\bT$ with error parameter $1/2$ returns a vector $v$ such that
    \[
	\iprod{v,a_i}^2 \ge 1-16\alpha.
	\]
\end{lemma}
\begin{proof}
    Without loss of generality, let $i:=1$ so that $\iprod{b,a_i}^2 = 1-\eta\ge 0.99$ (henceforth, we use $i$ as an ordinary index).
    We have that
    \begin{align*}
	\E[(Ax)^{\tensor 4}](u\tensor u)
	&= p \sum_{i}  \iprod{u,a_i}^2 a_i^{\tensor 2} \\
	&\qquad + \sum_{i\neq j} \E[x_i^2 x_j^2]\left(\iprod{u,a_j}^2\cdot a_i^{\tensor 2} + \iprod{u,a_i}\iprod{u,a_j} (a_i \tensor a_j + a_j \tensor a_i)\right)
    \end{align*}
    Define $M_u$ to be the reshaping of $\E[(Ax)^{\tensor 4}](u\tensor u)$ to a $d^2 \times d^2$ matrix.
    We must understand the spectrum of $M_u$, and for now we turn our attention to the second sum.
Splitting the second sum into distinct parts, we have by the orthonormality of the $a_i$ that
    \begin{align*}
\sum_{i\neq j} \E[x_i^2 x_j^2]\cdot \iprod{u,a_j}^2\cdot a_ia_i^\top
	&\preceq \max_{i\neq j}\E[x_i^2x_j^2]\cdot \sum_{i\neq j} a_ia_i^\top
	~\preceq~ \alpha p \cdot \Id,
    \end{align*}
    where the last line is by our assumption on $\E[x_i^2x_j^2]$.
    Finally, for any $w \in \R^d$,
    \begin{align*}
	& w^\top\left(\sum_{i\neq j} \E[x_i^2 x_j^2]\cdot  \iprod{u,a_i}\iprod{u,a_j} (a_i a_j^\top + a_ja_i^\top)\right)w\\
	&\qquad = \sum_{i\neq j} \E[x_i^2x_j^2] \cdot 2\iprod{u,a_i}\iprod{u,a_j}\iprod{w,a_i}\iprod{w,a_j}\\
	\intertext{Applying Cauchy-Schwarz and pulling out the maximum multiplier,}
	&\qquad \le 2\max_{i\neq j}(E[x_i^2x_j^2])\cdot\left( \sum_{i\neq j}\iprod{u,a_i}^2\iprod{w,a_j}^2\right)^{1/2}\left(\sum_{i\neq j}\iprod{u,a_j}^2\iprod{w,a_i}^2\right)^{1/2}\\
	\intertext{Now noticing that the two parenthesized terms are actually identical, then adding a positive quantity and factoring,}
	&\qquad \le 2\max_{i\neq j}(E[x_i^2x_j^2])\cdot\left( \left(\sum_{i}\iprod{u,a_i}^2\right)\left(\sum_j\iprod{w,a_j}^2\right)\right)
~=~ 2\max_{i\neq j}\E[x_i^2x_j^2]
	~\le~ 2p\alpha.
    \end{align*}
    An identical proof, up to signs, gives us a lower bound of $2p\alpha$.

    Therefore,
    \[
	\frac{1}{p} M_u = \sum_{i}\iprod{u,a_i}^2 a_ia_i^\top + E,
    \]
For a matrix $E$ with $\|E\| \le 4\alpha$.

It remains to argue that the top eigenvector of $M_u$ is $a_1$.
We have that
    \begin{align*}
	\|p^{-1}M_u\|
	&\ge a_1^\top M_u a_1\\
	& a_1^\top \sum_{i}\iprod{u,a_i}^2 a_ia_i^\top a_1 + a_1^\top E a_1\\
	&\ge 1-\eta - 4\alpha,
    \end{align*}
    where the last line follows from the orthonormality of the $a_i$ and our bound on $\|E\|$.
    Meanwhile, for any unit vector $w \in \R^d$ and any $\epsilon < 1-2\eta$,
    \begin{align*}
	w^\top\left(p^{-1}M - \epsilon \cdot a_1a_1^\top\right)w
	&= (1-\eta - \epsilon)\iprod{a_1,w}^2 + \sum_{i>1}\iprod{u,a_i}^2\iprod{a_i,w}^2 + w^\top E w\\
	\intertext{and since $\max_{i>1} \iprod{a_i,u}^2 \le \eta$,}
	&\le (1-\eta-\epsilon)\iprod{a_1,w}^2 + \eta\cdot (1-\iprod{a_1,w}^2) + 4 \alpha\\
	&\le 1 - \eta - \epsilon + 4\alpha,
    \end{align*}
    where the last line follows because $w$ is a unit vector, and we chose $\epsilon$ so that $1-\eta-\epsilon > \eta$.
Therefore
    \[
    \|p^{-1}M_u - \epsilon a_1a_1^\top\| \le \|p^{-1} M_u\| - \epsilon +8\alpha.
	\]
	Applying \pref{fact:topeig}, we conclude that if $v$ is the top eigenvector of $M_u$, then $\iprod{v,a_1}^2 \ge \frac{\epsilon - 8\alpha}{\epsilon} \ge 1 - \frac{8\alpha}{\epsilon}$.
	Since we can choose $\epsilon = \frac{1}{2}$ and still have that $\epsilon < 0.98 < 1-2\eta$, the conclusion follows.
\end{proof}

\subsection{From independent columns to orthonormal columns}\label{sec:indtoorth}

We now use standard techniques to prove that one can reduce from a dictionary with independent columns to a dictionary with orthogonal columns, given sufficiently many samples.
\begin{proof}[Proof of \pref{lem:indtoorth}]
    Note that the expected covariance matrix of the samples is equal to a scaled version of the covariance matrix,
    \[
	\E[(Ax)(Ax)^\top] = \E[x_1]^2 \cdot \Sigma.
    \]
Since we have assumed that $\|x\|_2^2 \le n$, the Frobenius norm of $(Ax)(Ax)^\top$ is bounded by $n$ for every sample, and $\E[((Ax)(Ax)^\top)^2] \le n^2$.
    By applying a matrix Bernstein inequality (see e.g. \cite{DBLP:journals/focm/Tropp12}), we have that so long as we have $m \ge \tO((n/\beta)^2)$ samples, the empirical covariance matrix $\hat\Sigma = \frac{1}{m} \sum_{i=1}^m y^{(i)}(y^{(i)})^\top$ will approximate $\Sigma$ within $\beta$ in spectral norm.

So given sufficiently many samples, we can compute a good spectral approximation of $\Sigma^{-1/2}$, $\hat \Sigma^{-1/2}$ with
    \[
	\left \|\hat\Sigma^{-1/2} - \Sigma^{-1/2}\right\| \le \epsilon.
	\]
    Assuming access to such a $\hat \Sigma^{-1/2}$, we can transform $A$ to $\tilde{A} = \hat\Sigma^{-1/2}A$.
    Now, for matrices $X,Y,Z$ of suitable dimensions,
    \[
	YXY - ZXZ = \frac{1}{2}((Y-Z)X(Y+Z) + (Y+Z)X(Y-Z))\mper
    \]
Applying this to $\tilde A \tilde A^\top - \Id = \hat\Sigma^{-1/2} \Sigma \hat\Sigma^{-1/2} - \Sigma^{-1/2} \Sigma \Sigma^{-1/2}$,
    \begin{align*}
	\left\|\tilde{A}\tilde{A}^\top - \Id\right\|
	&\le \|\hat\Sigma^{-1/2} - \Sigma^{-1/2}\|\cdot \|\Sigma\| \cdot (\|\hat\Sigma^{-1/2}\| + \|\Sigma^{-1/2}\|)\\
	&\le \epsilon \cdot \|\Sigma\|\cdot (2+\epsilon)\|\Sigma^{-1/2}\|\\
	&\le O\left(\epsilon\frac{\lambda_{\max}(\Sigma)}{\lambda_{\min}(\Sigma)^{1/2}}\right)\mper
    \end{align*}
    Defining $\eta \defeq\|\tilde{A}\tilde{A}^\top - \Id\|$, we have
    that $\tilde{A}\tilde{A}^\top = (1\pm \eta)\Id$, and the columns of $\tilde{A}$ are near-orthonormal.
    Similarly, we can transform our samples
    \[
	y^{(i)} = Ax^{(i)} \to \tilde{y}^{(i)} = \hat\Sigma^{-1/2} Ax^{(i)}.
	\]
    Now, define $S_{diff} = \left((\hat\Sigma^{-1/2})^{\tensor 2} - (\Sigma^{-1/2})^{\tensor 2} \right)$ and $S_{sum} = \left((\hat\Sigma^{-1/2})^{\tensor 2} + (\Sigma^{-1/2})^{\tensor 2} \right)$.
    We can factor the difference
    \begin{align*}
	\left\|\E[(\tilde{A}x)^{\tensor 4} - (\Sigma^{-1/2}Ax)^{\tensor 4}]\right\|
	&= \left\|\frac{1}{2}S_{diff}\E[(Ax)^{\tensor 4}]S_{sum}+\frac{1}{2}S_{sum}\E[(Ax)^{\tensor 4}]S_{diff}\right\|\\
	&\le \|S_{sum}\|\cdot \|S_{diff}\| \cdot \left\|\E[(Ax)^{\tensor 4}]\right\|\\
	&\le \left(\|\hat\Sigma^{-1/2}\|^2+ \|\Sigma^{-1/2}\|^2\right)\cdot \|S_{diff}\| \cdot \left\|\E[(Ax)^{\tensor 4}]\right\|\\
	&\le (2+\eps)\|\Sigma^{-1}\|\cdot \|S_{diff}\| \cdot \left\|\E[(Ax)^{\tensor 4}]\right\|
    \end{align*}
    Applying the identity
    \[A^{\tensor 2} - B^{\tensor 2} = \frac{1}{2}(A-B)\tensor (A+B) + \frac{1}{2} (A-B)\tensor (A+B),
    \]
    to $S_{diff}$, we can get that
\begin{align*}
	\left\|\E[(\tilde{A}x)^{\tensor 4} - (\Sigma^{-1/2}Ax)^{\tensor 4}]\right\|
	&\le (2+\epsilon)\|\Sigma^{-1}\|\cdot\|\hat \Sigma^{-1/2} - \Sigma^{-1/2}\|\left(\|\hat \Sigma^{-1/2}\|+\| \Sigma^{-1/2}\|\right)\left\|\E[(Ax)^{\tensor 4}]\right\|\mcom\\
	&\le 9\cdot \|\Sigma^{-3/2}\|\cdot\epsilon\left\|\E[(Ax)^{\tensor 4}]\right\|\mper
    \end{align*}
    Since we can choose the number of samples so as to make this last quantity as small as we would like, as a function of the condition number, and then appealing to \pref{fact:orthog}, the reduction is complete.
\end{proof}

\subsection{Sample complexity bounds}
\label{sec:samplecomplex}

Below is our bound on the sample complexity of approximating the $4$th moment tensor, which we believe may be loose.

\begin{proposition}
    Given samples of the form $y^{(i)} = Ax^{(i)}$ for $x^{(i)}\sim \cD$, as long as $\beta \ge \E[x_i^8]$ dominates the expectation of any other order-$8$ monomial in $x$, and any monomial with odd multiplicity has expectation $0$, and the entries of $x$ are bounded by $\kappa$, then with high probability given $m \ge \tO(\max\{\beta n^{3}, (\kappa n)^2\})$ samples,
    \[
	\left\|\frac{1}{m}\sum_{i=1}^m (y^{(i)})^{\tensor 4} - \E_{x\sim \cD}[(Ax)^{\tensor 4}] \right\| \le o(1).
	\]
\end{proposition}
\begin{proof}
Our matrix has the form
    \[
	M = \sum_{ijk\ell} x_ix_jx_kx_\ell \cdot (a_i \tensor a_j)(a_k\tensor a_\ell)^\top,
    \]
And the $a_i$ are orthonormal, so
    \[
	\E[MM^\top] = \sum_{\substack{i,j,i',j'\\k,\ell}} \E[x_ix_jx_{i'}x_{j'}x_k^2x_\ell^2] \cdot (a_i \tensor a_j)(a_{i'}\tensor a_{j'})^\top.
    \]
This is because of the orthonormality of the $a_i$, which guarantees that terms in the product $MM^\top$ in which we have an inner product between two non-identical vectors drop out to $0$.

    If we define $A$ to be the $d^2 \times n^2$ matrix whose columns are the Kronecker products $a_i \tensor a_j$ for all $i,j\in[n]$, and if for each pair $k,\ell \in [n]$ we define $X^{(k,\ell)}$ be the $n^2 \times n^2$ matrix whose $(i,j),(i',j')$th entry is $\E[x_ix_jx_{i}x_{j'}x_k^2x_\ell^2]$, we can realize $E[MM^\top]$ as
    \[
	\E[MM^\top] = A \left(\sum_{k,\ell} X^{(k,\ell)}\right) A^\top.
    \]
Because we assumed that $E[x_ix_jx_{i'}x_{j'}x_k^2x_\ell^2] = 0$ unless every index appears with even multiplicity, the entry of $X^{(k,\ell)}$ will be $0$.
    This happens only on the diagonal, unless $i'=j'$ and $i=j$, or for $X^{(k,\ell)}$ in the intersection of the $(k,\ell)$th row and the $(\ell,k)$th column and the $(\ell,k)$th row and the $(k,\ell)$th column.
    So we split each $X^{(k,\ell)}$ into a diagonal part $D^{(k,\ell)}$, an intersection part corresponding to the $(k,\ell)$ and $(\ell,k)$ intersections $C^{(k,\ell)}$, and the rest of the off-diagonal part $R^{(k,\ell)}$.
and it follows that
    \[
	\|\E[MM^\top]\| \le n^2 \|A\|^2 \max_{k,\ell} \left(\left\|D^{(k,\ell)}\right\| + \left\|C^{k,\ell}\right\|+ \left\|R^{(k,\ell)}\right\|\right).
    \]
Because every entry is bounded by $\E[x_i^8] \le \beta$, and the $D^{(k,\ell)}$ are diagonal, the $\|D\|$ term contributes $\beta$.
    The $C$ matrices have Frobenius norm $2\beta$, and the $R$ matrices have only $n^2$ nonzero entries, so $\|R\|_F \le \beta n$.
    Therefore,
    \[
	\|\E[MM^\top]\| \le 3\beta n^3.
    \]
For each sample $x^{(i)}$, we have that $\|\frac{1}{m}(Ax^{(i)})^{\tensor 4}\|_F \le \frac{\kappa^2n^2}{m}$, and we have by the above reasoning that $\|\E[\frac{1}{m^2}(Ax^{(i)})^{\tensor 4}(Ax^{(i)})^{\tensor 4}] \| \le 3\E[x_i^8] \frac{n^3}{m^2}$.
    So the variance of the empirical $4$-tensor is $\sqrt{n^3/m}$, and the absolute bound on the norm of any summand is $n^2/m$.
    Applying a matrix Bernstein inequality (see e.g. \cite{DBLP:journals/focm/Tropp12}), we have that as long as we have $m \gg \max\{\kappa^2n^2\log n,\beta n^{3}\log n\}$ samples, with high probability we approximate $\E[(Ax)^{\tensor 4}]$ within spectral norm $o(1)$.
\end{proof}

\bibliography{bib/mathreview,bib/dblp,bib/scholar,bib/custom}

\newcommand{\etalchar}[1]{$^{#1}$}
\providecommand{\bysame}{\leavevmode\hbox to3em{\hrulefill}\thinspace}
\providecommand{\MR}{\relax\ifhmode\unskip\space\fi MR }
\providecommand{\MRhref}[2]{%
  \href{http://www.ams.org/mathscinet-getitem?mr=#1}{#2}
}
\providecommand{\href}[2]{#2}
\begin{thebibliography}{YWHM08}

\bibitem[AAJ{\etalchar{+}}14]{DBLP:conf/colt/AgarwalA0NT14}
Alekh Agarwal, Animashree Anandkumar, Prateek Jain, Praneeth Netrapalli, and
  Rashish Tandon, \emph{Learning sparsely used overcomplete dictionaries},
  {COLT}, {JMLR} Workshop and Conference Proceedings, vol.~35, JMLR.org, 2014,
  pp.~123--137.

\bibitem[ABGM14]{DBLP:journals/corr/AroraBGM14}
Sanjeev Arora, Aditya Bhaskara, Rong Ge, and Tengyu Ma, \emph{More algorithms
  for provable dictionary learning}, CoRR \textbf{abs/1401.0579} (2014).

\bibitem[AFH{\etalchar{+}}12]{DBLP:conf/nips/AnandkumarFHKL12}
Anima Anandkumar, Dean~P. Foster, Daniel~J. Hsu, Sham Kakade, and Yi{-}Kai Liu,
  \emph{A spectral algorithm for latent dirichlet allocation}, {NIPS}, 2012,
  pp.~926--934.

\bibitem[AGH{\etalchar{+}}14]{DBLP:journals/jmlr/AnandkumarGHKT14}
Animashree Anandkumar, Rong Ge, Daniel~J. Hsu, Sham~M. Kakade, and Matus
  Telgarsky, \emph{Tensor decompositions for learning latent variable models},
  Journal of Machine Learning Research \textbf{15} (2014), no.~1, 2773--2832.

\bibitem[AGHK13]{DBLP:conf/colt/AnandkumarGHK13}
Animashree Anandkumar, Rong Ge, Daniel~J. Hsu, and Sham Kakade, \emph{A tensor
  spectral approach to learning mixed membership community models}, {COLT},
  {JMLR} Workshop and Conference Proceedings, vol.~30, JMLR.org, 2013,
  pp.~867--881.

\bibitem[AGMM15]{DBLP:conf/colt/AroraGMM15}
Sanjeev Arora, Rong Ge, Tengyu Ma, and Ankur Moitra, \emph{Simple, efficient,
  and neural algorithms for sparse coding}, {COLT}, {JMLR} Workshop and
  Conference Proceedings, vol.~40, JMLR.org, 2015, pp.~113--149.

\bibitem[AGMR16]{DBLP:journals/corr/AroraGMR16}
Sanjeev Arora, Rong Ge, Tengyu Ma, and Andrej Risteski, \emph{Provable learning
  of noisy-or networks}, CoRR \textbf{abs/1612.08795} (2016).

\bibitem[BCMV14]{DBLP:conf/stoc/BhaskaraCMV14}
Aditya Bhaskara, Moses Charikar, Ankur Moitra, and Aravindan Vijayaraghavan,
  \emph{Smoothed analysis of tensor decompositions}, {STOC}, {ACM}, 2014,
  pp.~594--603.

\bibitem[BKS15]{DBLP:conf/stoc/BarakKS15}
Boaz Barak, Jonathan~A. Kelner, and David Steurer, \emph{Dictionary learning
  and tensor decomposition via the sum-of-squares method}, {STOC}, {ACM}, 2015,
  pp.~143--151.

\bibitem[EA06]{EladA2006}
Michael Elad and Michal Aharon, \emph{Image denoising via sparse and redundant
  representations over learned dictionaries}, Image Processing, IEEE
  Transactions on \textbf{15} (2006), no.~12, 3736--3745.

\bibitem[EP07]{ArgyriouEP06}
Andreas Argyriou~Theodoros Evgeniou and Massimiliano Pontil, \emph{Multi-task
  feature learning}, Advances in Neural Information Processing Systems 19:
  Proceedings of the 2006 Conference, vol.~19, MIT Press, 2007, pp.~41--48.

\bibitem[GO09]{goldstein2009split}
Tom Goldstein and Stanley Osher, \emph{The split bregman method for
  l1-regularized problems}, SIAM journal on imaging sciences \textbf{2} (2009),
  no.~2, 323--343.

\bibitem[Har70]{harshman1970foundations}
Richard~A Harshman, \emph{Foundations of the parafac procedure: Models and
  conditions for an" explanatory" multi-modal factor analysis}.

\bibitem[HK13]{MR3385380-Hsu13}
Daniel Hsu and Sham~M. Kakade, \emph{Learning mixtures of spherical
  {G}aussians: moment methods and spectral decompositions},
  I{TCS}'13---{P}roceedings of the 2013 {ACM} {C}onference on {I}nnovations in
  {T}heoretical {C}omputer {S}cience, ACM, New York, 2013, pp.~11--19.
  \MR{3385380}

\bibitem[HM16]{DBLP:conf/nips/HazanM16}
Elad Hazan and Tengyu Ma, \emph{A non-generative framework and convex
  relaxations for unsupervised learning}, {NIPS}, 2016, pp.~3306--3314.

\bibitem[HP14]{DBLP:conf/nips/HardtP14}
Moritz Hardt and Eric Price, \emph{The noisy power method: {A} meta algorithm
  with applications}, {NIPS}, 2014, pp.~2861--2869.

\bibitem[HSSS15]{DBLP:journals/corr/HopkinsSSS15}
Samuel~B. Hopkins, Tselil Schramm, Jonathan Shi, and David Steurer,
  \emph{Speeding up sum-of-squares for tensor decomposition and planted sparse
  vectors}, CoRR \textbf{abs/1512.02337} (2015).

\bibitem[HSSS16]{DBLP:conf/stoc/HopkinsSSS16}
\bysame, \emph{Fast spectral algorithms from sum-of-squares proofs: tensor
  decomposition and planted sparse vectors}, {STOC}, {ACM}, 2016, pp.~178--191.

\bibitem[LCC07]{DBLP:journals/tsp/LathauwerCC07}
Lieven~De Lathauwer, Jos{\'{e}}phine Castaing, and Jean{-}Fran{\c{c}}ois
  Cardoso, \emph{Fourth-order cumulant-based blind identification of
  underdetermined mixtures}, {IEEE} Trans. Signal Processing \textbf{55}
  (2007), no.~6-2, 2965--2973.

\bibitem[LMV96]{DBLP:conf/eusipco/LathauwerMV96}
Lieven~De Lathauwer, Bart~De Moor, and Joos Vandewalle, \emph{Blind source
  separation by simultaneous third-order tensor diagonalization}, {EUSIPCO},
  {IEEE}, 1996, pp.~1--4.

\bibitem[MLB{\etalchar{+}}08]{MairalLBHP2008}
Julien Mairal, Marius Leordeanu, Francis Bach, Martial Hebert, and Jean Ponce,
  \emph{Discriminative sparse image models for class-specific edge detection
  and image interpretation}, Computer Vision--ECCV 2008, Springer, 2008,
  pp.~43--56.

\bibitem[MR05]{DBLP:conf/stoc/MosselR05}
Elchanan Mossel and S{\'{e}}bastien Roch, \emph{Learning nonsingular
  phylogenies and hidden markov models}, {STOC}, {ACM}, 2005, pp.~366--375.

\bibitem[MRBL07]{RanzatoBL2007}
Y~Marc'Aurelio~Ranzato, Lan Boureau, and Yann LeCun, \emph{Sparse feature
  learning for deep belief networks}, Advances in neural information processing
  systems \textbf{20} (2007), 1185--1192.

\bibitem[MSS16]{DBLP:conf/focs/MaSS16}
Tengyu Ma, Jonathan Shi, and David Steurer, \emph{Polynomial-time tensor
  decompositions with sum-of-squares}, {FOCS}, {IEEE} Computer Society, 2016,
  pp.~438--446.

\bibitem[OF97]{olshausen1997sparse}
Bruno~A Olshausen and David~J Field, \emph{Sparse coding with an overcomplete
  basis set: A strategy employed by v1?}, Vision research \textbf{37} (1997),
  no.~23, 3311--3325.

\bibitem[{Oli}10]{zbMATH05946839}
Roberto~I. {Oliveira}, \emph{{Sums of random Hermitian matrices and an
  inequality by Rudelson.}}, {Electron. Commun. Probab.} \textbf{15} (2010),
  203--212 (English).

\bibitem[Roc76]{Rock76}
R.~Tyrrell Rockafellar, \emph{Monotone operators and the proximal point
  algorithm}, SIAM Journal on Control and Optimization \textbf{14} (1976),
  no.~5, 877--898.

\bibitem[Tro12]{DBLP:journals/focm/Tropp12}
Joel~A. Tropp, \emph{User-friendly tail bounds for sums of random matrices},
  Foundations of Computational Mathematics \textbf{12} (2012), no.~4, 389--434.

\bibitem[YWHM08]{YangWHY2008}
Jianchao Yang, John Wright, Thomas Huang, and Yi~Ma, \emph{Image
  super-resolution as sparse representation of raw image patches}, Computer
  Vision and Pattern Recognition, 2008. CVPR 2008. IEEE Conference on, IEEE,
  2008, pp.~1--8.

\end{thebibliography}
\addreferencesection
\bibliographystyle{amsalpha}

\appendix
\section{Useful Tools}\label{app:tools}

\begin{lemma}[Concentration of random tensor contractions \cite{DBLP:conf/focs/MaSS16}]\label{lem:gflat}
    Let $g$ be a standard Gaussian vector in $\R^{k}$, $g \sim \cN(0,\Id_{k})$.
    Let $A$ be a tensor in $(\R^k)\otimes (\R^{\ell}) \otimes (\R^{m})$, and call the three modes of $A$ $\alpha,\beta,\gamma$ respectively.
    Let $A_i$ be a $\ell \times m$ slice of $A$ along mode $\alpha$.
    Then,
    \[
	\Pr\left[\left\|\sum_{i=1}^k g_iA_{i}\right\| \ge t \cdot \max\left\{\left\|A_{\{\alpha\beta\}\{\gamma\}}\right\|, \left\|A_{\{\alpha\gamma\}\{\beta\}}\right\|\right\}\right] \le (m+\ell)\exp\left(-\frac{t^2}{2}\right)\mper
    \]
\end{lemma}
\begin{proof}
    We compute the expectation and variance of our matrix,
    \[
	\E_g\left[\sum_{i=1}^k g_iA_{i}\right] = 0,
	\qquad \text{and}\qquad
	\left\|\Var_g\left[\sum_{i=1}^k g_iA_{i}\right] \right\| = \max\left\{\left\|\sum_{i=1}^k A_iA_i^{\top}\right\|, \left\|\sum_{i=1}^k A_i^{\top}A_i\right\|\right\},
    \]
    The two variance terms correspond to $\|A_{\{\alpha\beta\}\{\gamma\}}\|^2$ and $\|A_{\{\alpha\gamma\}\{\beta\}}\|^2$ respectively.
    We can now apply concentration results for matrix Gaussian series to conclude the proof \cite{zbMATH05946839}.
\end{proof}

The following lemma states that distances can only decrease under projections to a convex set, and is well-known (see e.g. \cite{Rock76}).
\begin{lemma}\label{lem:proxproj}
    Let $\cC\subset \R^n$ be a closed convex set, and let $\Pi:\R^n \to \cC$ be the projection operator onto $\cC$ in terms of norm $\|\cdot\|_2$, i.e. $\Pi(x) \defeq \argmin_{c \in \cC} \|x-c\|_2$.
    Then for any $x,y \in \R^n$,
    \[
	\|x-y\|_2 \ge \|\Pi(x) - \Pi(y)\|_2.
    \]
\end{lemma}
\begin{proof}
    If we let $D_x = x - \Pi(x)$, $D_y = y - \Pi(y)$,
    \begin{align*}
	\|x-y\|^2
	&= \|D_x - D_y + \Pi(x) - \Pi(y)\|^2\\
	&= \|D_x- D_y\|^2 + \|\Pi(x) - \Pi(y)\|^2 + 2\iprod{D_x - D_y, \Pi(x) - \Pi(y)}
    \end{align*}
    Now the conclusion will follow from the fact that
    \[
\iprod{D_x - D_y, \Pi(x) - \Pi(y)} \ge 0.
    \]
    This is because, by definition of $\Pi$,
    \[
	\Pi(x)
	= \argmin_{c \in \cC} \|x-c\|_2^2
	= \argmin_{p \in \R^n} \frac{1}{2}\|x-p\|_2^2 + \Ind_{\cC}(p)\mcom
    \]
where $\Ind_{\cC}(\cdot)$ is the convex function defined to be $\infty$ on elements not in $\cC$ and $0$ otherwise.
    From the strong convexity of the last expression the projection is unique.
    From the optimality conditions, it follows that $\Pi(x)$ is the unique point $p \in \R^n$ such that $x - p \in \partial \Ind_{\cC}(p)$, where $\partial \Ind_{\cC}(p)$ is the set of subgradients of $\Ind_{\cC}$ at $p$.

    By definition of the subgradient and by the convexity of $\Ind_{\cC}$, for any $p,q \in \R^n$ and for $g_p \in \partial\Ind_{\cC}(p)$,$g_q \in \partial\Ind_{\cC}(q)$,
    \begin{align*}
	\Ind_{\cC}(p) + \iprod{g_p,q-p} &\le \Ind_{\cC}(x)\\
	 \iprod{g_p ,q-p} &\le \Ind_{\cC}(q) - \Ind_{\cC}(p)\\
	 -\iprod{g_q ,q-p} &\le -\Ind_{\cC}(q) + \Ind_{\cC}(p)\\
	 \iprod{g_p - g_q, p-q} &\ge 0\mper
    \end{align*}
    Now, taking $p = \Pi(x)$ and $q = \Pi(y)$, we have that $D_x = x - \Pi(x) \in \partial \Ind_{\cC}(\Pi(x))$, and
    $D_y = y - \Pi(y) \in \partial \Ind_{\cC}(\Pi(y))$, so from the above,
    \begin{align*}
	\iprod{D_x - D_y, \Pi(x) - \Pi(y)}
	&\ge 0,
    \end{align*}
    as desired.
\end{proof}

\begin{fact}\label{fact:topeig}
    Let $v \in \R^n$, and suppose that $\|M - vv^{\top}\| \le \|M\| - \epsilon\|v\|^2$.
    Then if $u,w$ are the top unit left- and right-singular vectors of $M$,
    \[
	\iprod{u,v}^2 \ge \epsilon \cdot \|v\|^2 \quad \text{or}\quad \iprod{w,v}^2 \ge \epsilon \cdot \|v\|^2
    \]
\end{fact}
\begin{proof}
   We have that
    \begin{align*}
	\|M\| - \epsilon \|v\|^2
	\ge |u^{\top}(M-vv^{\top})w|
	\ge |u^\top M w| - |u^\top v v^\top w|
	&= \|M\| - |\iprod{u,v}\iprod{w,v}|,
    \end{align*}
    where the second inequality is the triangle inequality.
    Rearranging, the conclusion follows.
\end{proof}

\begin{fact}\label{fact:projcont}
    If $A$ is an $n \times n$ symmetric matrix with eigendecomposition $\sum_{i\in[n]} \lambda_i u_i u_i^\top$ for orthonormal $u_1,\ldots,u_n \in \R^n$ and eigenvalues $\lambda_1\ge \cdots \ge \lambda_n$, then the projection of $A$ to the PSD cone is equal to $\sum_{i \in [n]} \Ind[\lambda_i \ge 0]\cdot \lambda_i u_i u_i^\top$.
\end{fact}
\begin{proof}
Let $\hat A = A + B$ be the projection (in Frobenius norm) of $A$ to the PSD cone.
Because $u_1,\ldots,u_n$ are orthonormal, we may choose an orthonormal basis $V=\{v_i\}_{i=1}^{n^2}$ for $\R^{n^2}$ that includes $v_1 = u_1 \tensor u_1,v_2 = u_2 \tensor u_2,\ldots, v_n = u_n \tensor u_n$, as the first $n$ basis vectors.
Now, viewing $A,B$ as vectors in $\R^{n^2}$, we can write $A = \sum_{i=1}^n \lambda_i v_i$ for $\lambda_i$ the eigenvalues of $A$, and write $B = \sum_{i=1}^{n^2} \beta_i v_i$ for some scalars $\beta_1,\ldots,\beta_{n^2}$.

For any eigenvector $u_i$ of $A$, we have that $u_i^\top \hat A u_i = \iprod{v_i, A + B} = \lambda_i + \beta_i$ by the orthonormality of the $v_i$.
Therefore, if $\lambda_i < 0$ we must have $\beta_i \ge |\lambda_i|$, since $\hat A$ is PSD.
We also have that $\|A - \hat A\|_F^2 = \|B\|_F^2 = \sum_{i=1}^{n^2} \beta_i^2$, so $B = \sum_{i=1}^n \mathbb{I}[\lambda_i < 0] \cdot |\lambda_i|\cdot u_i u_i^\top$ minimizes the Frobenius norm of the difference, which (after checking to see that $A + B$ has all non-negative eigenvalues) concludes the proof.

\end{proof}

\end{document}